\documentclass{article}       
\usepackage{booktabs}       %
\usepackage{amsfonts}       %
\usepackage{nicefrac}       %
\usepackage{microtype}      %
\usepackage{amsmath, amsfonts, amssymb, amsthm} 
\usepackage{algorithm}
\usepackage{algorithmicx, algpseudocode}
\usepackage{mathtools}
\usepackage{bbm}
\usepackage{subcaption}
\usepackage{thm-restate}
\usepackage{xcolor}
\usepackage{mathrsfs}

\usepackage{tikz}
\usetikzlibrary{calc,matrix,decorations.markings,decorations.pathreplacing}

\DeclareMathAlphabet{\mathpzc}{OT1}{pzc}{m}{it}

\definecolor{puorange}{rgb}{0.80,0.20,0}
\definecolor{bluegray}{rgb}{0.04,0,0.7}
\definecolor{darkbrown}{rgb}{0.40,0.2,0.05}
\usepackage[colorlinks,citecolor=bluegray,linkcolor=darkbrown,urlcolor=blue,breaklinks]{hyperref}

\usepackage{scalerel}
\usetikzlibrary{svg.path}

\definecolor{orcidlogocol}{HTML}{A6CE39}
\tikzset{
  orcidlogo/.pic={
    \fill[orcidlogocol] svg{M256,128c0,70.7-57.3,128-128,128C57.3,256,0,198.7,0,128C0,57.3,57.3,0,128,0C198.7,0,256,57.3,256,128z};
    \fill[white] svg{M86.3,186.2H70.9V79.1h15.4v48.4V186.2z}
                 svg{M108.9,79.1h41.6c39.6,0,57,28.3,57,53.6c0,27.5-21.5,53.6-56.8,53.6h-41.8V79.1z M124.3,172.4h24.5c34.9,0,42.9-26.5,42.9-39.7c0-21.5-13.7-39.7-43.7-39.7h-23.7V172.4z}
                 svg{M88.7,56.8c0,5.5-4.5,10.1-10.1,10.1c-5.6,0-10.1-4.6-10.1-10.1c0-5.6,4.5-10.1,10.1-10.1C84.2,46.7,88.7,51.3,88.7,56.8z};
  }
}

\newcommand\orcidicon[1]{\href{https://orcid.org/#1}{\mbox{\scalerel*{
\begin{tikzpicture}[yscale=-1,transform shape]
\pic{orcidlogo};
\end{tikzpicture}
}{|}}}}

\usepackage[misc]{ifsym}

\newtheorem{theorem}{Theorem}[section]
\newtheorem{proposition}[theorem]{Proposition}
\newtheorem{definition}[theorem]{Definition}

\newtheorem{example}[theorem]{Example}
\theoremstyle{remark}

\usepackage{fullpage}
\usepackage{natbib}

\newcommand*{\Avgfeat}{\Phi}

\newcommand*{\entreg}{\nu}

\newcommand*{\lagr}{\gamma}

\newcommand*{\n}{n}

\DeclareMathOperator{\Proj}{P}

\newcommand*{\pinv}{+}

\newcommand*{\reg}{\zeta}
\newcommand*{\regtwo}{\lambda}
\newcommand*{\Regtwo}{\Lambda}

\newcommand*{\x}{x}

\newcommand*{\y}{y}

\newcommand{\obj}{F}

\newcommand{\ulr}{\text{ulr}}

\DeclareMathOperator*{\Tr}{tr}
\DeclareMathOperator*{\tr}{tr}
\DeclareMathOperator*{\diag}{diag}

\newcommand*{\mbr}{\mathbb{R}}

\newcommand*{\mcc}{\mathcal{C}}

\newcommand*{\mck}{\mathcal{K}}
\newcommand*{\mcl}{\mathcal{L}}
\newcommand*{\mcr}{\mathcal{R}}
\newcommand*{\mcs}{\mathcal{S}}
\newcommand*{\mcu}{\mathcal{U}}

\newcommand*{\mbone}{\mathbbm{1}}

\newcommand{\reals}{\mathbb{R}}

\newcommand{\ones}{\mathbbm{1}}
\newcommand{\id}{\operatorname{ I}}
\newcommand{\Id}{\operatorname{ I}}

\allowdisplaybreaks

\begin{document}

\title{Discriminative Clustering with Representation Learning\\with any Ratio of Labeled to Unlabeled Data}

\author{Corinne Jones\textsuperscript{1}, \quad
        Vincent Roulet\textsuperscript{2},  \quad
        Zaid Harchaoui\textsuperscript{2} \\
        {\small {}\textsuperscript{1}Swiss Data Science Center, Ecole polytechnique f\'ed\'erale  de Lausanne, 1015 Lausanne, Switzerland}\\
        {\small {}\textsuperscript{2}Department of Statistics, University of Washington, Seattle, WA 98195, USA}\\
}

\date{}

\maketitle

\begin{abstract}
We present a discriminative clustering approach in which the feature representation can be learned from data and moreover leverage labeled data. Representation learning can give a similarity-based clustering method the ability to automatically adapt to an underlying, yet hidden, geometric structure of the data. The proposed approach augments the DIFFRAC method with a representation learning capability, using a gradient-based stochastic training algorithm and an optimal transport algorithm with entropic regularization to perform the cluster assignment step. The resulting method is evaluated on several real datasets when varying the ratio of labeled data to unlabeled data and thereby interpolating between the fully unsupervised regime and the fully supervised regime. The experimental results suggest that the proposed method can learn powerful feature representations even in the fully unsupervised regime and can leverage even small amounts of labeled data to improve the feature representations and to obtain better clusterings of complex datasets. 
\end{abstract}

\section{Introduction}
Similarity-based clustering methods have been successfully applied in a number of applications, from computational chemistry to spectral imaging~\citep{hennig2015handbook,megaman}. The data is assumed to live in a space equipped with a dot product (or, alternatively, a distance function), which relies on a linear or nonlinear representation mapping. A popular example is the Gaussian radial basis function kernel, used for instance for spectral clustering, which implicitly yields a nonlinear representation mapping in a Hilbert space parameterized by the kernel parameters~\citep{vonluxburg2007,vanengelen2020}. 

When the representation mapping is parameterized by a small number of parameters, the parameters can then be selected using a cross-validation procedure with labeled data~\citep{vonluxburg2007,hennig2015handbook}. Recent progress in representation learning~\citep{oliver2018}, which can be seen as a generalized form of metric learning, allows us to revisit this question by parameterizing the transformation using deep artificial neural networks~\citep{ardila2019,thickstun2018,li2018} and to leverage the potential of labeled data. 

Equipping a clustering method with representation learning, and framing an objective that allows one to incorporate labeled data, that is, cluster assignments known beforehand, poses several challenges. The iterative algorithms involved in $k$-means (alternating minimization) or spectral clustering (spectral factorization) are not easily compatible with gradient-based optimization algorithms commonly used to learn feature representations in classical supervised learning such as multi-class classification. The issue arises from the different natures of the corresponding objectives that would need to be reconciled in a common framework. 
 Moreover, the information gained from assignments known beforehand (from labeled data) must inform both the representation learning and the resulting clustering, which again requires a framework in order to be conducted in a principled manner. 

In this paper we propose a discriminative clustering approach equipped with representation learning. The proposed approach is applicable when there is only unlabeled data, or some unlabeled data and some labeled data, or only labeled data. A precise and unique objective function allows us to recover or approach classical unsupervised learning and supervised learning objectives. Indeed, the objective naturally reduces to that of a clustering problem when we have no training set labels and that of a classification problem when we have all of the training set labels. Moreover, due to its simplicity, this setup can be extended. For example, indirect constraints on the labels, such as requiring two unlabeled observations to have different labels, can be readily incorporated. Such constraints can be useful if, e.g., someone labeling the data knew two observations should have different labels but did not know the correct label for each observation.

After reviewing related work on unsupervised and semi-supervised learning in Section~\ref{sec:related_work}, we present the framework in Section~\ref{sec:mfam_learning}. We then focus on a specific objective in Section~\ref{sec:framework}, showing that our proposed objective is smoother than a straightforward alternative. We address how to optimize the objective function in Section~\ref{sec:opt}. Optimizing over the labels requires care, and for this we present a novel algorithm based on a convex relaxation of the problem. Finally, we demonstrate the proposed approach in Section~\ref{sec:experiments}, showing that our method, called XSDC, is competitive with existing methods that are less flexible in their usage.

\section{Related Work}
\label{sec:related_work}

In this work we propose an approach that (1) clusters data regardless of the ratio of labeled to unlabeled data; and (2) learns a feature representation using the data at hand. 
We first survey prior approaches to clustering that work with varying levels of supervision. We then describe recent approaches to learning feature representations when no labeled data is available and when some labeled data is available. 

\paragraph{Semi-supervised clustering.} Intuitively, there are two main ways of developing a clustering algorithm that can work with both labeled and unlabeled data. First, one could modify a supervised classification method so that it can incorporate unlabeled data. Such modifications come in different flavors, including adding a penalty to a supervised learning objective to encourage similar inputs to be close together in feature space \citep{belkin2006,bachman2014,kamnitsas2018,iscen2019}, adding a penalty to encourage high-confidence outputs \citep{grandvalet2004}, or rounding outputs to obtain pseudo-labels \citep{lee2013,berthelot2019}. Other approaches add a supervised loss to an unsupervised loss \citep{beyer2019}. 
Alternatively, one could modify a clustering algorithm in order to incorporate labeled data. Approaches of this kind include constrained clustering using a $k$-means formulation~\citep{basu2002semi,bilenko:etal:2004,yoder2017semi} and generalizations thereof \citep{xu2009,white2012}, and fractionally-supervised classification  based on a Gaussian mixture model \citep{vrbik2015}. See the survey of \citet{oliver2018} and the books of \citet{chapelle2010,bouveyron2019}, and \citet{vanengelen2020} for an overview of semi-supervised algorithms.

The approach we take is based on DIFFRAC \citep{bach2007}, which falls in the former class of methods. DIFFRAC is a discriminative clustering method, that is, an unsupervised clustering method built off a supervised classification method. In the case of DIFFRAC, the supervised classification method is regularized least squares. 
In order to avoid trivial solutions, cluster size constraints are enforced. Various extensions of DIFFRAC have also been considered in the literature \citep{joulin2012,flammarion2017}. The advantage of the objective introduced by \citet{bach2007} is that it allows one to easily incorporate additional information about the clustering problem. Namely, it paved the way to several popular 
weakly supervised learning techniques developed by \citet{bojanowski2014, bojanowski2015} and \citet{alayrac2016} for computer vision problems.

\paragraph{Representation learning.} A large number of representation learning methods exist. Here we survey representation learning methods that are unsupervised and work with only unlabeled data or that are semi-supervised and work with both labeled and unlabeled data.

Most unsupervised deep feature learning methods can be broadly classified into one of two categories: methods that optimize a surrogate loss, often based on known structure in the data; and methods that directly optimize a loss function of interest. Early examples of the former set of methods include auto-encoders, which attempt to reconstruct the input observations through a deep network \citep{lecun1987,goodfellow2016}. Other more recent examples attempt to approximate a kernel at each layer of a network~\citep{bo2011,mairal2014,daniely2017}. 
Most recently, many papers have been taking advantage of structure in the data. This includes training to distinguish between multiple views of images or patches and other images or patches \citep{wang2015,dosovitskiy2016,sermanet2018,bachman2019}, learning to predict the relative locations of patches in images \citep{doersch2015,noroozi2016}, and predicting color from grayscale images \citep{zhang2016}. It also includes learning to distinguish segments within time series or patches within images, or to predict future observations in time series \citep{hyvarinen2016,lowe2019}. 
A downside to these latter approaches is the focus on achieving state-of-the-art results on domain-specific tasks in computer vision and signal processing at the expense of the conciseness of the formulation.

The second category of unsupervised methods typically alternately optimizes the parameters of the network and the labels or cluster assignments of the observations. In this thread, several papers alternate between obtaining assignments or soft assignments and optimizing the parameters of a loss function aimed at creating well-separated clusters \citep{xie2016,yang2016,ghasedi2017,hausser2017}. 
In contrast, \citet{bojanowski2017} randomly generate outputs and then alternately optimize over the parameters of the model and the assignment of labels to outputs. The most direct approach may be that of \citet{caron2018}, who alternately cluster the data to obtain pseudo-labels and take steps to optimize the multinomial logistic loss on the observations with the given pseudo-labels. A drawback of these approaches is the design of an \textit{ad-hoc} objective not clearly related to objectives commonly used in unsupervised clustering or supervised classification, or the combined use of two different objectives, one for optimizing the network and one for clustering. 

The category of semi-supervised representation learning methods includes a number of the semi-supervised clustering methods discussed above.  \citet{lee2013,kamnitsas2018,berthelot2019,beyer2019}, and \citet{iscen2019} all propose methods for learning features in the presence of unlabeled data. This category also includes approaches that learn a feature representation in an unsupervised manner before fine-tuning with labeled data \citep[e.g.,][]{wu2018}.  A downside to these approaches is that they either do not use a single objective function or they are not designed to work in the purely unsupervised setting. In this paper we build our formulation on an objective that encompasses the three settings of learning with unlabeled data only, learning with labeled and unlabeled data, and learning with labeled data only.

\paragraph{Relation to existing methods.} 
This work may be viewed as an extension of DIFFRAC \citep{bach2007} in which the feature representation is also learned from data. As argued by~\citet{daniely2017}, a feature representation defined by a deep network can sometimes be related to an approximation of a feature map associated with a composition of reproducing kernels. From this viewpoint, the approach in this paper can also be interpreted as learning a reproducing kernel, i.e., a similarity measure, acting on pairs of examples. Learning a similarity measure for the purpose of clustering was first explored by~\citet{meila2005} and \citet{bach2006}, whose focus was on learning kernel parameters from labeled data. Our approach can be seen as more general in that any differentiable feature representation defined as a chained composition of parameterized mappings can be learned from data for the purpose of clustering using the optimization and labeling  algorithms we propose. \citet{law2017} proposed a deep learning approach, but it was also purely supervised. 

In addition to using deep networks, we improve upon the work of~\citet{bach2007} by proposing a simplified convex relaxation of the labeling subproblem. This relaxation allows us to handle several types of constraints on the labels. The corresponding subproblem is similar to the problem~\citet{zass2006} solved to find a doubly stochastic matrix for use in spectral clustering. The problem we consider includes an additional regularization term that makes the problem strictly convex and enforces non-negativity of the minimizer. The labeling procedure we propose recovers the Sinkhorn-Knopp algorithm \citep{sinkhorn1967,peyre2019} when there is no labeled data and the sizes of the clusters are assumed to be known. 

After the first version of this work was completed,~\citet{asano2020} developed a similar approach aimed at representation learning for unsupervised cluster analysis, with a focus on computer vision problems such as image classification and object detection. In contrast to our approach, their method is based on a batch optimization algorithm. Moreover, their formulation is parameterized with respect to the label (assignment) matrix rather than the equivalence matrix.

\begin{figure*}[h!]
\begin{subfigure}{.33\linewidth}
\centering
\begin{tikzpicture}
\tikzset{square matrix/.style={
    matrix of nodes,
    column sep=-\pgflinewidth, row sep=-\pgflinewidth,
    nodes={draw,
      text height=#1/2+0.75ex,
      text depth=#1/2-0.75ex,
      text width=#1,
      align=center,
      inner sep=0pt
    },
  },
  square matrix/.default=0.6cm
}
\matrix(m)[square matrix]{
1 & 0 \\
1 & 0 \\
0 & 1 \\
1 & 0 \\
0 & 1 \\
};
\normalsize
  \foreach \x[count=\i from 0] in {1,...,5}\node[left] at (m-\x-1.west) {$\i$};
  \foreach \y[count=\j from 0] in {1,2}\node[above] at (m-1-\y.north) {$\j$};
\node[above, yshift=2.0cm] at (m) {Cluster index};
\node[left, xshift=-1.35cm,yshift=1.5cm,rotate=90] at (m) {Observation index};
\end{tikzpicture}
\caption{\label{fig:assign_mat}Assignment matrix \\ \centering $Y$}
\end{subfigure}%
\begin{subfigure}{.33\linewidth}
\centering
\begin{tikzpicture}
\tikzset{square matrix/.style={
    matrix of nodes,
    column sep=-\pgflinewidth, row sep=-\pgflinewidth,
    nodes={draw,
      text height=#1/2+0.75ex,
      text depth=#1/2-0.75ex,
      text width=#1,
      align=center,
      inner sep=0pt
    },
  },
  square matrix/.default=0.6cm
}
\matrix(m)[square matrix]{
1 & 1 & 0 & 1 & 0\\
1 & 1 & 0 & 1 & 0\\
0 & 0 & 1 & 0 & 1\\
1 & 1 & 0 & 1 & 0\\
0 & 0 & 1 & 0 & 1\\
};
\normalsize
  \foreach \x[count=\i from 0] in {1,...,5}\node[left] at (m-\x-1.west) {$\i$};
  \foreach \y[count=\j from 0] in {1,...,5}\node[above] at (m-1-\y.north) {$\j$};
\node[above, yshift=2.0cm] at (m) {Observation index};
\node[left, xshift=-2.25cm,yshift=1.5cm,rotate=90] at (m) {Observation index};
\end{tikzpicture}%
\caption{\label{fig:equiv_mat} Equivalence matrix \\ \centering $M=YY^T$}
\end{subfigure}
\begin{subfigure}{.33\linewidth}
\centering
\begin{tikzpicture}
\scriptsize
\tikzset{square matrix/.style={
    matrix of nodes,
    column sep=-\pgflinewidth, row sep=-\pgflinewidth,
    nodes={draw,
      text height=#1/2+0.75ex,
      text depth=#1/2-0.75ex,
      text width=#1,
      align=center,
      inner sep=0pt
    },
  },
  square matrix/.default=0.6cm
}
\matrix(m)[square matrix]{
1/3 & 1/3 & 0 & 1/3 & 0\\
1/3 & 1/3 & 0 & 1/3 & 0\\
0 & 0 & 1/2 & 0 & 1/2\\
1/3 & 1/3 & 0 & 1/3 & 0\\
0 & 0 & 1/2 & 0 & 1/2\\
};
\normalsize
  \foreach \x[count=\i from 0] in {1,...,5}\node[left] at (m-\x-1.west) {$\i$};
  \foreach \y[count=\j from 0] in {1,...,5}\node[above] at (m-1-\y.north) {$\j$};
\node[above, yshift=2.0cm] at (m) {Observation index};
\node[left, xshift=-2.25cm,yshift=1.5cm,rotate=90] at (m) {Observation index};
\end{tikzpicture}%
\caption{\label{fig:norm_equiv_mat}Normalized equivalence matrix \\ \centering $\tilde M=YY^T(YY^T)^\pinv$}
\end{subfigure}
\caption{\label{fig:clustering_matrices} Three different ways of representing a clustering.}
\end{figure*}

\section{Learning with any Level of Supervision}
\label{sec:mfam_learning}
We first describe a framework allowing us to circumscribe a family of methods whose objective can be conveniently reformulated in terms of the equivalence matrix. 
We show how, in this framework, one can easily incorporate information from labeled data if labeled data is available. We shall build off our approach to develop a generalization of the DIFFRAC method by equipping it with a representation learning capability. The proposed approach, as well as the companion algorithm, shall be referred to as \emph{XSDC}. The acronym stands for ``X-Supervised Discriminative Clustering'', where ``X'' can be ``un'', ``semi'' or ``-'', highlighting that all regimes of supervision 
(as one varies the ratio of labeled  to unlabeled data) are covered by the approach.

\subsection{Clustering methods based on equivalence matrices}
Consider observations $x_1, \dots, x_n \in\mbr^d$, each belonging to one of $k$ (unknown) clusters. 
There are several ways to represent the assignments of the observations $x_1,\dots, x_n$ to clusters \citep[see, e.g., ][]{zha2001,bach2006}. First, one can use the assignment matrix $Y\in\{0,1\}^{n\times k}$, where $Y_{i, \cdot}\eqqcolon y_i$ is a one-hot cluster assignment vector for observation $i$. Alternatively, one can use the equivalence matrix $M=YY^T$. In this matrix, entry $(i,j)$ is 1 if observations $i$ and $j$ belong to the same cluster and is 0 otherwise. Finally, one can use the normalized equivalence matrix $\tilde M=YY^T(YY^T)^\pinv$, where $M^\pinv$ denotes the pseudo-inverse of the matrix $M$. In this matrix entry $(i,j)$ is $1/n_{i}$ if observations $i$ and $j$ belong to the same cluster and is 0 otherwise, where $n_{i}$ is the number of elements in the cluster observations $i$ and $j$ belong to.  An example of each of these representations is depicted in Figure~\ref{fig:clustering_matrices}.

Now let $S\in\mbr^{n\times n}$ be a given similarity matrix derived from the observation matrix $X=[x_1,\dots, x_n]^T$, where $S_{ij}$ is the similarity between observations $i$ and $j$. Consider the problem of assigning observations to clusters. Intuitively, we want to maximize the similarity of points within each cluster. 
In other words, denoting $\langle A, B\rangle = \Tr(A^\top B)$, we might consider solving $\max_{\tilde M} \langle S, \tilde M \rangle$ or $\max_M \langle S, M \rangle$ subject to the constraint that each observation lies in exactly one cluster and $\tilde M$ or $M$ is a (normalized) equivalence matrix. 
In each case trivial solutions can exist (e.g., in the second case, if $S$ is strictly positive, then a trivial solution assigns all observations to the same cluster). 
To avoid such solutions we can add constraints on the cluster sizes. As we will see shortly, for particular choices of $S$, these two problems can lead to previously-established clustering algorithms, such as $k$-means and DIFFRAC \citep{macqueen1967,bach2007}. This intuition motivates the following family of clustering problems that we study in this section. 
To align with traditional clustering objectives we write the problem in terms of minimization rather than maximization.

\begin{table*}
\caption{\label{tab:m_family_methods} Examples of clustering methods belonging to the family of clustering methods from Definition~\ref{def:eq_fam_methods}. If not specified in the text, the notations used are the same as those in the references.}
\begin{center}
\small
\begin{tabular}{c|cccccc}
Algorithm & $\Psi_\theta(X)$ & $\Gamma_\theta(X)$  & $\alpha$ & $\beta$ & $\lagr_1$ & $\lagr_2$ \\
\toprule
Correlation clustering \citep{swamy2004} & $(w_{\text{out}}-w_{\text{in}})^T$ & $\Id_n$ & 1 & 0 & 0 & 0 \\
DIFFRAC \citep{bach2007} & 
$A_\lambda(X)$ & $\Id_n$ & 1  & 0 & 1 & 1\\
DIFFRAC-cosegmentation \citep{joulin2010} & $A_\lambda(X) + \mu/n L(X)$ & $\Id_n$ & 1 & 0 & 1 & 1 \\
$k$-means \citep{macqueen1967} & $-XX^T$ & $\Id_n$ & 1 & 1 & 0 & 0\\
Kernel $k$-means \citep{scholkopf1998}  & $-K$& $\Id_n$ & 1 & 1 & 0 & 0\\
Spectral clustering (Balanced cut) \citep{wu1993}& $L$ & $\Id_n$ & 1  & 0 & 1 & 1\\
Spectral clustering (NCut) \citep{shi2000} & $D^{-1/2}LD^{-1/2}$ & $D^{1/2}$ & 1  & 1 & 1 & 0\\
Spectral clustering (Ratio Cut) \citep{hagen1992} & $L$ & $\Id_n$ & 1  & 1 & 1 & 0\\
Stochastic block model$^a$ \citep{jalali2016} & $-A\log\frac{p_{\tau}(1-q)}{1-p_{\tau}q} - \Id_n\log\frac{1-p_{\tau}}{1-q}$& $\Id_n$ & 1 & 0 & 0 & 0 \\
\bottomrule
\end{tabular}
\end{center}
\footnotesize
$^a$The stochastic block model formulation assumes that the $p_i$'s and $q$ are known and $p_{\tau(i,j)}=p_\tau$ for all $i,j$, i.e., it is a homogeneous stochastic block model.
\end{table*}

\begin{definition}
\label{def:eq_fam_methods}
Let $\Psi_\theta:\mbr^{n\times d}\to\mbr^{n\times n}$ and $\Gamma_\theta:\mbr^{n\times d}\to\mbr^{n\times n}$ be functions parameterized by $\theta\in\mbr^{d_\theta}$ for some $d_\theta$. 
Defining $\Xi_\theta(X, Y) = \Gamma_\theta(X)YY^T\Gamma_\theta(X)$, a family of clustering problems is given by
\begin{align}
\min_{Y}\quad &\left\langle \Psi_\theta(X), \Xi_\theta(X, Y)^\alpha{\Xi_\theta(X, Y)^\pinv}^\beta\right\rangle \label{eq:m_family_obj}\\
\text{subject to} \quad &Y\mbone_k=\mbone_n, \ \lagr_1\left(Y^T\mbone_n-n_{\min}\mbone_k\right) \geq 0, \ \lagr_2\left(Y^T\mbone_n-n_{\max}\mbone_k\right) \leq 0, \ y_{ij}\in\{0,1\}\quad \forall \; i,j\nonumber\;,
\end{align}
where $\alpha,\beta,\lagr_1,\lagr_2\in\{0,1\}$ and $n_{\min}, n_{\max}> 0$ are the minimum and maximum allowable cluster sizes.
\end{definition}
Note that the objective can be written exclusively in terms of the equivalence matrix $M=YY^T$. 
Since we are minimizing rather than maximizing the objective, we can think of $\Psi_\theta(X)$ as a dissimilarity matrix on the observations parametrized by $\theta$. Usually $\Gamma_\theta(X)$ is the identity. 
However, in normalized cut spectral clustering 
where it is the degree matrix 
we can think of $\Gamma_\theta(X)$ as reweighting the entries of $M$ according to how important each observation is. Next we show how we can recover $k$-means, DIFFRAC, and normalized cut spectral clustering, given particular choices of $\Psi_\theta, \Gamma_\theta, \alpha,\beta,\lagr_1,$ and $\lagr_2$. Table~\ref{tab:m_family_methods} summarizes how these methods and some other common clustering algorithms fit into this family.

\begin{example}[$k$-means]
Define\: $\Psi_\theta(X)=-XX^T$ and $\Gamma_\theta(X)=\Id_n$, and let $\alpha=\beta=1$ and $\lagr_1=\lagr_2=0$. The resultant problem in the family from Definition~\ref{def:eq_fam_methods} is given by
\begin{align*}
\min_{Y}\quad &\left\langle -XX^T, YY^T{(YY^T)^\pinv}\right\rangle \\
\text{subject to} \quad &Y\mbone_k=\mbone_n, \  y_{ij}\in\{0,1\}\quad \forall \; i,j\;.
\end{align*}
Adding a term $\langle X, X\rangle$ to the objective, which does not affect the minimizer, and using the fact that for a matrix $Z$, $ZZ^\pinv=ZZ^T(ZZ^T)^\pinv$ \citep[p. 35]{lutkepohl1996}, we obtain
\begin{align*}
\left\langle XX^T, \id_n-YY^T{(YY^T)^+}\right\rangle &= \left\langle XX^T, \id_n-YY^\pinv\right\rangle = \|X-YY^\pinv X\|_F^2 = \min_{\mu\in\mbr^{k\times d}} \|X-Y\mu\|_F^2\;.
\end{align*}
Note that each row $\ell$ of the minimizer $\mu^\star$ contains the mean of the rows $X_{i, \cdot}$ of $X$ belonging to cluster $\ell$, i.e., the mean of the $X_{i, \cdot}$'s where $Y_{i,\ell} = 1$. The overall problem can then be written as
\begin{align*}
\min_{Y, \mu}\quad &\|X-Y\mu\|_F^2\\
\text{subject to} \quad& Y\mbone_k = \mbone_n, \ y_{ij}\in\{0,1\}\quad \forall \; i,j\;,
\end{align*} 
which is precisely the $k$-means problem.
\end{example}

\begin{example}[DIFFRAC with cluster size constraints]
Define $\Psi_\theta(X)=A_\lambda(X)$, where $A_\lambda(X)$ is given by $A_\lambda(X) \coloneqq \lambda\Pi_n \left(\Pi_n X X^T \Pi_n+ n\lambda\id\right)^{-1} \Pi_n$ and $\Pi_n$ is a centering matrix, $\Pi_n = \id_n - \ones_n\ones_n^T/n$. Furthermore, define $\Gamma_\theta(X)=\Id_n$. Let $\alpha=1, \beta=0$, and $\lagr_1=\lagr_2=1$. These choices of the parameters in the family from Definition~\ref{def:eq_fam_methods} lead to the problem

\begin{samepage}
\begin{align*}
\min_{Y}\quad &\langle A_\lambda(X), YY^T\rangle \\
\text{subject to} \quad& Y\mbone_k = \mbone_n, \  Y^T\mbone_n \geq \n_{\min}\mbone_k\, \  Y^T\mbone_n \leq \n_{\max}\mbone_k, \ y_{ij}\in\{0,1\}\quad \forall \; i,j\;.
\end{align*}
\end{samepage}%
This is precisely the DIFFRAC problem of \citet[equation 2]{bach2007} with cluster size constraints. \citet{bach2007} showed that this problem is the same as the following ridge regression problem, 
when also optimizing over $Y$:
\begin{align}
\label{eq:orig_diffrac}
\min_{\substack{Y\in\mcc^D_Y,\\ W\in\mbr^{d\times k},b\in\mbr^k}}\frac{1}{n}\sum_{i=1}^n \|y_i-(W^Tx_i + b)\|_2^2 + \lambda\|W\|_F^2\;,
\end{align}
where
$\mcc^D_Y = \{Y \in \{0,1\}^{n\times k}:Y\mbone_k=\mbone_n,  \n_{\min}\mbone_k \leq Y^T\mbone_n \leq \n_{\max}\mbone_k\}$ is the constraint set on the labels. 
\end{example}

\begin{example}[Normalized cut spectral clustering]
Let  $S_\theta\in\mbr^{n\times n}$ be a non-negative, symmetric similarity matrix derived from $X$ (e.g., $(S_\theta)_{ij} = \exp(-\|x_i-x_j\|^2/(2\theta^2))$). Given such a matrix $S_\theta$, which we will henceforth denote by $S$, define the degree matrix $D=\diag([D_i]_{i=1}^n)$ where $D_i = \sum_{j=1}^n S_{ij}$ for all $i$ and the Laplacian matrix $L=D-S$. Assume the degree $D_{ii}$ for each observation $i$ is strictly positive such that the degree matrix $D$ is invertible. Define $\Psi_\theta(X)=D^{-1/2}LD^{-1/2}$,  $\Gamma_\theta(X)=D^{1/2}$, and $\Xi_\theta(X,Y)=D^{1/2}YY^TD^{1/2}$, and let $\alpha=\beta=1$ and $\lagr_1=\lagr_2=0$. These choices lead to the problem from Definition~\ref{def:eq_fam_methods} given by

\begin{samepage}
\begin{align*}
\min_{Y}\quad &\left\langle D^{-1/2}LD^{-1/2}, \Xi_\theta(X,Y){\Xi_\theta(X,Y)^\pinv}\right\rangle \\
\text{subject to} \quad &Y\mbone_k=\mbone_n. \ Y^T\mbone_n \geq \mbone_k, \  y_{ij}\in\{0,1\}\quad \forall \; i,j\;.
\end{align*}
\end{samepage}%
The additional constraint $Y^T\mbone_n \geq \mbone_k$ ensures that there is at least one point per cluster. Without this constraint, the solution to the problem would assign all points to the same cluster. 
Using the fact that for a matrix $Z$, $ZZ^\pinv=ZZ^T(ZZ^T)^\pinv$ \citep[p. 35]{lutkepohl1996}, we can rewrite the objective as 
\begin{align*}
\tr\left[D^{-1/2}LD^{-1/2}\Xi_\theta(X,Y){\Xi_\theta(X,Y)^\pinv}\right] 
=\;& \tr\left((D^{1/2}Y)^{\pinv}D^{-1/2}LY\right) \\
=\;& \tr\left((Y^TDY)^{-1}Y^TLY\right) 
=\; \sum_{j=1}^k \frac{Y_{\cdot, j}^TLY_{\cdot, j}}{Y_{\cdot, j}^TDY_{\cdot, j}}\;.
\end{align*}

Let $\mathscr{C}=\{C_1,\dots, C_k\}$ define a clustering, where each $C_j$ is a set containing the indices of the observations in cluster $j$. Moreover, denote the sum of the degrees of nodes in a set $C$ by $\text{Vol}(C)$, the volume of $C$.  Spectral clustering traditionally makes use of the concept of a ``cut'' between two sets $C$ and $C'$, defined to be the sum of the similarities between elements in set $C$ and in set $C'$ \citep{vonluxburg2007,meila2016}:
 \begin{align*}
\text{Cut}(C,C') \coloneqq \sum_{i\in C}\sum_{j\in C'} S_{ij}\;.
\end{align*}
With this definition, we may rewrite the objective as
\begin{align*}
\sum_{j=1}^k \frac{Y_{\cdot, j}^TLY_{\cdot, j}}{Y_{\cdot, j}^TDY_{\cdot, j}} &= \sum_{j=1}^k\sum_{j\neq j'} \frac{\text{Cut}(C_j, C_{j'})}{\text{Vol}(C_j)}\;,
\end{align*}
where the last line follows from observing that $\text{Vol}(C_j)=Y_{\cdot, j}^TDY_{\cdot, j}$ and $\sum_{j'\neq j}\text{Cut}(C_j,C_{j'}) = Y_{\cdot, j}^T(D-S)Y_{\cdot, j}$ \citep{xing2003}. Therefore, the problem may be written as
\begin{align*}
\min_{\mathscr{C}=\{C_1,\dots, C_k\}} \quad &\sum_{j=1}^k\sum_{j'\neq j} \frac{\text{Cut}(C_j, C_{j'})}{\text{Vol}(C_j)}\\
\text{subject to} \quad& C_j\cap C_{j'} = \emptyset \ \forall\; j\neq j', \ \cup_{j=1}^k C_j = \{1,\dots,n\}\;,
\end{align*}
which is precisely the multi-way normalized cut clustering problem.
\end{example}

\subsection{Incorporating labeled data}
A natural way to take labeled data into account in the objective from Definition~\ref{def:eq_fam_methods} is to add additional constraints on the labels. Specifically, we now consider the case where the cluster label for each observation $i$, $y_i^\star$, may or may not be observed. Here the cluster label is represented using a dummy variable or one-hot encoding. We denote by $\mcs$  the set of indices corresponding to the labeled data and by $\mcu$ the set of indices corresponding to the unlabeled data. In this case, the constraint set on the label matrix $Y$ becomes $\mcc_Y = \{Y \in \{0,1\}^{n\times k}:Y\mbone_k=\mbone_n, \lagr_1\left(Y^T\mbone_n-n_{\min}\mbone_k\right) \geq 0, \lagr_2\left(Y^T\mbone_n-n_{\max}\mbone_k\right) \leq 0, y_i=y_i^\star \: \mbox{ for } \:  i \in \mcs\}$. We can translate this constraint set to the following constraint set on the equivalence matrix $M$: $\mcc_M = \{M\in\{0,1\}^{n\times n}: \exists Y\in\mcc_Y \text{ s.t. } M=YY^T\}$. In addition to optimizing over the entries of $M$, we can also consider optimizing over the parameters $\theta$ from Definition~\ref{def:eq_fam_methods}.

\begin{figure*}[t!]
 \vspace{-0.32cm}
\centering
\includegraphics[trim={0.25cm 0cm 0.3cm 0cm},clip,width=0.8\linewidth]{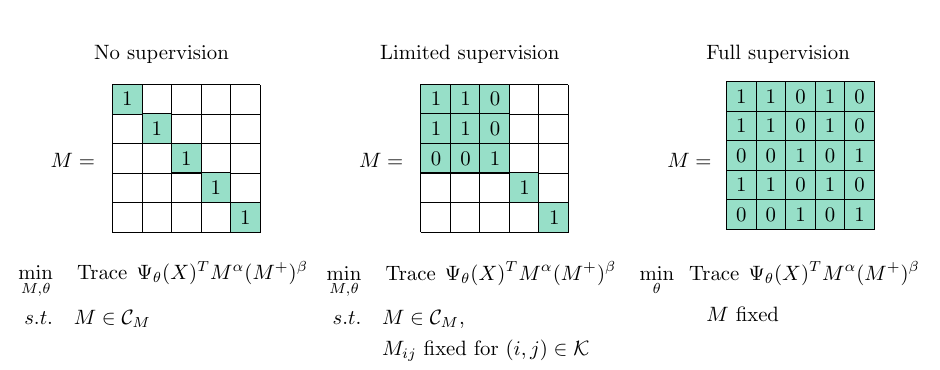}
\caption{\label{fig:supervisions} Example equivalence matrix $M$ and objective function for varying levels of supervision. For simplicity we set $\Gamma_\theta(X)=\Id_n$ in the objective functions.}
\end{figure*}

The advantage of this problem formulation is that it captures three regimes: the unsupervised regime, in which clustering is performed to learn the equivalence matrix $M$ in addition to the parameters $\theta$; the supervised regime, in which supervised training is performed to learn the parameters $\theta$; and the semi-supervised regime, in which a combination of clustering and supervised training are performed to learn the unknown elements of $M$, in addition to the parameters $\theta$. Given a specific clustering objective and an optimization algorithm, we may therefore proceed with training regardless of the amount of labeled data. Figure~\ref{fig:supervisions} displays examples of the equivalence matrix $M$ and the problem in the cases of no labeled data, some labeled data, and fully-labeled data.

\section{Extension of the DIFFRAC Objective}
\label{sec:framework}
\newcommand{\batchsize}{n_b}
In the remainder of this paper we focus on the extension of the DIFFRAC objective to equip it with the capability to learn a feature representation from unlabeled data and any amount of additional labeled data.

\subsection{Problem formulation}
\label{sec:deep_cluster}
In order to learn a feature representation, we propose transforming the input to the original DIFFRAC objective \eqref{eq:orig_diffrac} via the mapping defined by a deep network $\phi_V:\mbr^d\to\mbr^D$. The deep network is assumed here to be implemented within a differentiable programming framework. That is, the deep network is assumed to be amenable to automatic differentiation with respect to any (subset) of its parameters. We aim to use both the labeled and unlabeled data to learn (a) the parameters $V_\ell$ at each layer $\ell=1,2\dots, m$ of $\phi_V$, where $V = \{V_1,\dots, V_m\}$; and (b) the parameters $W\in\mbr^{D\times k}$ and $b\in\mbr^{k}$ of the classifier on the output features $\phi_V(x_i)$, $i=1,\dots, n$. Note that since the features are of dimension $D$, the dimension of $W$ has changed from $W\in\mbr^{d\times k}$ in Equation~\eqref{eq:orig_diffrac} to $W\in\mbr^{D\times k}$. For simplicity we will assume there exists a constant $B$ such that for all $V$ and for all $x\in\mbr^d$, $\|\phi_V(x)\|_2\leq B$, i.e., the network has bounded outputs.

To this end, we consider solving the problem
\begin{align}
\label{eq:training_obj_2}
\min_{\substack{Y\in\mcc_Y,\\ V,W,b}}\frac{1}{n}\sum_{i=1}^n \|y_i-(W^T\phi_V(x_i) + b)\|_2^2 + \mcr(V, W)\;,
\end{align}
where
$\mcc_Y = \{Y \in \{0,1\}^{n\times k}:Y\mbone_k=\mbone_n,\n_{\min}\mbone_k \leq Y^T\mbone_n \leq \n_{\max}\mbone_k,  y_i=y_i^\star \: \mbox{ for } \:  i \in \mcs\}$ is the constraint set on the labels  and
$\mcr(V, W) \coloneqq \reg\sum_{j=1}^m \|V_j\|_F^2  + \regtwo\|W\|_F^2$ contains the regularization terms.
Here $\reg\geq 0$ and $\regtwo\geq0$ are regularization parameters. We add a regularization penalty on the network parameters $V_j$ to promote smoothness of the learned network. In the following we denote simply $\phi_i(V) = \phi_V(x_i)$ and $\Phi(V) = (\phi_1(V),  \ldots, \phi_n(V))^T$.

 We shall in Section~\ref{sec:opt} present an algorithm to optimize this objective.

\paragraph{Comparison to reverse prediction objective.}
Following the terminology of \citet{xu2009}, the main component of our objective is regularized ``forward prediction'' least squares. 
By this phrase, we mean the linear prediction of $Y$ from $\Phi(V)$ incurring the objective value $\|Y-\Phi(V)W -\ones_n b^T\|_F^2/n+\regtwo\|W\|^2_F$. 
This is in contrast to the alternative ``reverse prediction'' least squares, where we would work with $\|\Phi(V)-YW\|_F^2/n$ instead. 
This arises for instance if we use a $k$-means-type objective. 
In both cases we can alternate between updating the parameters $V$ and $W$ and estimating the labels $Y$. 

One way in which we can compare the quality of the objectives generated by these two options for learning a representation is via their smoothness properties, i.e., their Lipschitz-continuity and the Lipschitz-continuity of their gradients. These control the step sizes of optimization methods; see \citet{bertsekas2016} and \citet{nesterov2018} for a discussion of the interplay between smoothness properties and rates of convergence. 
We now proceed to show that when fixing the labels $Y$ the forward prediction objective is smoother than the reverse prediction objective for appropriate choices of the regularization parameter $\regtwo$.

For both objectives, we consider fixed labels $Y \in \{0,1\}^{n\times k}$ with $Y \ones_k = \ones_n$. Moreover, for simplicity we will take $\reg=0$.
Consider the ``forward prediction'' objective from~\eqref{eq:training_obj_2}. Define the centering matrix $\Pi_n = \id_n - \ones_n\ones_n^T/n$. 
After minimizing over the intercept $b$, the problem may be written as
\begin{align}
	\min_V F_f(\Phi(V))\coloneqq 
&\min_{V,W} \frac{1}{n}\|\Pi_n[Y-\Phi(V)W]\|^2_F + \regtwo\|W\|^2_F 
	= \min_V \tr[Y Y^TA_\regtwo(\Phi(V))]\;, \label{eq:diffrac}
\end{align}
where $A_\regtwo(\Phi) =\regtwo \Pi_n \left(\Pi_n\Phi \Phi^T \Pi_n+ n\regtwo\id_n\right)^{-1} \Pi_n$. 

The corresponding ``reverse prediction'' problem is given by
\begin{align*}
\min_V F_r(\Phi(V)) &\coloneqq \min_{V,W} \frac{1}{n}\|\Phi(V)-YW\|^2_F =\min_V \frac{1}{n}\tr[(\id-P_Y)\Phi(V)\Phi(V)^T]\;,
\end{align*}
where $P_Y=Y(Y^TY)^{-1}Y^T$ is an orthonormal projector. 

To compare the smoothness with respect to any matrix $V_j$, $j=1,\dots, m$, it suffices to compute the smoothness with respect to $\Phi$. The next two propositions do that and suggest that the ``forward prediction'' objective is actually smoother than the ``reverse prediction'' objective. The proofs can be found in Appendix~\ref{app:conditioning}.

\begin{restatable}{proposition}{smoothnessobj}
\label{prop:smoothness_obj}
Let $\mathcal{Z}$ be the set of all possible feature matrices $\Phi\in\mbr^{n\times D}$. Assume there exists $B\in\mbr$ such that for all $\Phi\in\mathcal{Z}$, $\|\Phi\|_2\leq B$. 
Let $\rho_{\max}$  be a bound on the maximal fraction of points in a cluster, i.e., $\rho= n_{\max}/n$.
Then the Lipschitz constants of $F_f$ and $F_r$ with respect to the spectral norm can be estimated by
$L_f \coloneqq 2B \rho_{\max}/\left({n\regtwo}\right)$
and 
$L_r \coloneqq {2B}/{n}$
respectively. 
Hence, whenever  $\regtwo \geq \rho_{\max}$, we have  $L_f\leq L_r$.
\end{restatable}

\begin{restatable}{proposition}{smoothness}
\label{prop:smoothness}
Under the same assumption as Proposition~\ref{prop:smoothness_obj}, the Lipschitz constants of $\nabla F_f$ and $\nabla F_r$ with respect to the spectral norm can be estimated by 
$\ell_f \coloneqq 2\rho_{\max}/(n\regtwo) + 8B^2\rho_{\max}/(n\regtwo)^2$
and 
$\ell_r \coloneqq {2}/{n}$
respectively.  
Hence,  whenever $\regtwo \geq (\rho_{\max} + \sqrt{\rho_{\max}^2{+}16B^2\rho_{\max}})/2$, we have  $\ell_f\leq\ell_r$.
\end{restatable}

\subsection{Optimization algorithm}
\label{sec:opt}
The XSDC algorithm involves two main components: an optimization algorithm that leverages the algebraic structure of~\eqref{eq:training_obj_2}, and a cluster assignment algorithm that boils down to matrix balancing. 

The algorithm proceeds by using mini-batches. Below, we shall see that, with the square loss, we only ever need to work with the equivalence matrix $M\coloneqq YY^T$ rather than the label matrix $Y$ itself during training. Therefore, at each iteration we first estimate $M$ for a mini-batch given fixed $V,W$, and $b$. Then we update $V, W$, and $b$ for fixed $M$. The difficult part is estimating $M$. 

\paragraph{Stochastic training of parameters.}
To optimize over $V, W$, and $b$ we apply the ultimate layer reversal stochastic gradient optimization (ULR-SGO) method of \citet{jones2020}. Applied to~\eqref{eq:training_obj_2}, the resulting algorithm builds an outer loop around the optimization of $V$ and relegates to inner loops the optimization of the other variables. This stands in contrast to a direct approach using plain stochastic gradient optimization which would optimize all variables jointly regardless of their respective difficulty to be optimized. 

Concretely, denote the objective function \eqref{eq:training_obj_2} for fixed $Y$ by $\obj_{\ulr}(V, W, b)$. At each iteration, the algorithm computes $\hat\obj_\ulr(V)\coloneqq \min_{W,b}  \obj_{\ulr}(V, W, b)$, rewriting the objective exclusively in terms of $V$. Then it updates $V$ by taking one gradient step on $\hat\obj_\ulr(V)$. As long as $\obj_{\ulr}$ is twice differentiable and $\obj_{\ulr}$ viewed as a function of $W$ and $b$ is strongly convex, gradient descent on this objective converges to a stationary point and the resultant $\varepsilon-$ stationary points are $\varepsilon-$ stationary points of the original problem. If $\hat\obj_{\ulr}(V)$ is not available in closed form we may estimate it using a quadratic approximation of the loss around the current estimate of $V$. In addition, this method can also be applied on mini-batches and in the setting where $V$ is constrained.

The proposed optimization scheme has two main benefits. 
First, the focus on the optimization of $V$ facilitates the tuning of the step size along the iterations, keeping the number of parameters of the algorithm to a minimum.
Second, in the case of the square loss it allows us to work with the equivalence matrix $M= YY^T$ rather than the assignment matrix $Y$ during the alternating optimization. To see this, 
observe that from Equation~\eqref{eq:diffrac} we have
\begin{align*}
\hat\obj_{\ulr}(V) = 
\tr [MA_\regtwo(\Phi(V))] + R(V)\;,
\end{align*}
where $A_\regtwo$ is defined as in Equation~\eqref{eq:diffrac} and the regularization term is
$R(V) \coloneqq \reg\sum_{j=1}^m \|V_j\|_F^2\;.$
Since we only need to optimize over $M$ we can avoid dealing with the problem of there being many solutions $Y^\star$ caused by the optimal objective value being the same if the columns of $Y$ are permuted. 

\paragraph{Matrix balancing.}
Next, consider the objective function \eqref{eq:training_obj_2} when fixing $V,W$, and $b$ and optimizing over only the equivalence matrix $M=YY^T$. As shown in Proposition~\ref{prop:balancing_npcomplete} in Appendix~\ref{app:balancing}, this problem is NP-complete in general. 
Therefore, we consider a convex relaxation of it. We use an entropic regularizer $h(M)=\sum_{i,j=1}^n M_{ij}\log(M_{ij})$, which makes the objective strongly convex and enforces positivity of $M$. This regularizer appears in a Bregman divergence term $D_h(M;M_0)= h(M) - h(M_0) - \langle \nabla h(M_0), M - M_0 \rangle$, which can be used to ensure the output does not stray too far from an initial guess $M_0$. Specifically, we consider the problem
\begin{align}
\label{eq:balancing_obj}
\min_{M} \quad  & \frac{1}{2}\tr(MA) + \entreg D_h(M; M_0)  \\*
\mbox{subject to} \quad  & M_{ij}=m_{ij} \quad \forall \; (i,j)\in \mathcal{K}, \  n_{\min}\ones_n \leq  M \ones_n \leq n_{\max}\ones_n, \  n_{\min}\ones_n \leq  M^T \ones_n \leq n_{\max}\ones_n \nonumber \;,
\end{align}
where $m_{ij}$ for $i,j\in\mathcal{K}\coloneqq(\mcs\times\mcs) \cup \{(1,1),\dots, (n,n)\}$ represent the known entries of $M$ and $\entreg>0$ is a hyperparameter. Define $\tilde Q = \entreg^{-1}A - \log(M_0)$, $n_{\Delta} = (n_{\max} - n_{\min})/2$, and $n_{\Sigma} = (n_{\max} + n_{\min})/2$. Furthermore, let $\Regtwo=[\Regtwo_{ij}]_{i,j=1}^n$ with $\Regtwo_{ij}=\regtwo_{ij}$ if $(i,j)\in \mathcal{K}$ and $\Regtwo_{ij}= 0$ otherwise, where the $\regtwo_{ij}$'s are dual variables. After scaling the problem by $\entreg^{-1}$, the dual of this problem is then equivalent to 
\begin{align*}
\min_{\substack{a\in\mbr^n, b\in\mbr^n, c\in\mbr^n, \\ d\in\mbr^n, \regtwo\in\mbr^{|\mck|}}} \quad &  \ \exp(-a)^T\exp(-( \tilde Q+\Regtwo))\exp(-c) + n_{\Delta}(b+ d)^T \ones + n_{\Sigma} (a+c)^T \ones + \sum_{(i,j)\in\mck} \regtwo_{ij}m_{ij}\nonumber\\
\mbox{subject to} \quad  & b \geq |a|, \quad d \geq |c|\;. \nonumber
\end{align*}
Minimization in $b$ and $d$ can be performed analytically. We optimize over the remainder of the variables via alternating minimization. Defining $u=\exp(-a)$, $v=\exp(-c)$, and $N=\exp(-(\tilde Q + \Regtwo))$, the steps of the alternating minimization at iteration $t$ are given by
\begin{alignat*}{2}
	N^{(t)}_{ij} &= m_{ij}/\left(u_i^{(t-1)}v_j^{(t-1)}\right), \hspace{40pt}  \forall (i,j)\in\mck, \hspace{24pt}
	N^{(t)}_{ij} = \exp(-\tilde Q_{ij}), \hspace{78pt}  \forall (i,j)\notin\mck \\
	u^{(t)}_i &= \Proj_{\mathcal{B}_{\infty}(n_{\Sigma}, n_\Delta)}\left(\frac{{N^{(t)}_{i, \cdot}}^Tv^{(t-1)}}{{N^{(t)}_{i, \cdot}}^Tv^{(t-1)}}\right), \hspace{3pt} \forall i\in\{1,\dots, n\}, \hspace{10pt}
	v^{(t)}_i = \Proj_{\mathcal{B}_{\infty}(n_{\Sigma}, n_\Delta)}\left(\frac{{N^{(t)}_{\cdot, i}}^Tu^{(t)}}{{N^{(t)}_{ \cdot, i}}^Tu^{(t)}}\right), \quad  \forall i\in\{1,\dots, n\}, 
\end{alignat*} 
where $\Proj_{\mathcal{B}_{\infty}(x, R)}(y)$ denotes the projection on the unit $\ell_\infty$ ball centered at $x$ with radius $R$. 
This leads to Algorithm~\ref{alg:constrained_balancing}. In practice we find that 10 steps of the alternating minimization suffice. Note that in the case where the cluster sizes are predetermined and no labeled data exists, this reduces to the Sinkhorn-Knopp algorithm \citep{sinkhorn1967}. In Appendix~\ref{app:alt_label} we discuss an alternative relaxation of the labeling subproblem that was proposed by \citet{bach2007}.

\begin{algorithm}[H]
	\begin{algorithmic}[1]
		\caption{\label{alg:constrained_balancing} Matrix Balancing}
		\State{{\bfseries Input:} Matrix $A \in \reals^{n\times n}$\\
		\hspace{1.0cm} Matrix $\hat M \in\{0,1,?\}^{n \times n}$ encoding known \\ 
		\hspace{1.4cm} relations $m_{i,j} \in \{0,1\}$ with $(i,j)\in\mck$ 
 }
		\State{{\bfseries Hyperparameters}:\\ Minimum and maximum cluster sizes $n_{\min}$, $n_{\max}$, number of iterations $T$, entropic regularization $\entreg$ }
		\State{{\bf Initialize:} 
		$\tilde Q = \entreg^{-1}A-\log(\ones_n\ones_n^T/k)$\\
		 \hspace{1.5cm} $n_{\Delta} = (n_{\max} - n_{\min})/2$\\
		 \hspace{1.5cm} $n_{\Sigma} = (n_{\max} + n_{\min})/2$\\
		 \hspace{1.5cm} $u = v = \ones_n$
		 }
		 \For{$t=1,\ldots,T$}
		\State{$	N_{ij} \gets m_{ij}/(u_iv_j) \;, \hspace{2.18cm} (i,j)\in\mck$}
		\State{$	N_{ij} \gets \exp(-\tilde Q_{ij})\;,  \hspace{2.21cm} (i,j)\notin\mck$}
		\State{$ p_{v,i}\gets \Proj_{\mathcal{B}_{\infty}(n_{\Sigma}, n_\Delta)}\left(N_{i, 
		\cdot}^T v\right)\;, \hspace{0.96cm} i=1,\dots, n$}
		\State{$u_i \hspace{0.2cm} \gets p_{v,i} /(N_{i, \cdot}^Tv)\;, \hspace{2.18cm} i=1,\dots, n$}
		\State{$ p_{u,i} \gets \Proj_{\mathcal{B}_{\infty}(n_{\Sigma}, n_\Delta)}\left(N_{\cdot, i}^T u\right)\;, \hspace{0.93cm} i=1,\dots, n$}
		\State{$v_i \hspace{0.25cm} \gets p_{u,i} /(N_{\cdot, i}^Tu)\;, \hspace{2.135cm} i=1,\dots, n$}
		\EndFor
		\State{{\bfseries Output:} $M= \diag(u) N \diag(v)$}
	\end{algorithmic}
\end{algorithm}

 \begin{algorithm}[H]
\caption{\label{alg:xsdc} XSDC (when some labeled data exists)}
	\begin{algorithmic}[1]
		\State{{\bf Input:} Labeled data $X_\mcs, Y_\mcs$\\
		\hspace{1.0cm} Unlabeled data $X_\mcu$\\
		\hspace{1.0cm} Randomly initialized network parameters $V^{(0)}$\\
		\hspace{1.0cm} Number of iterations $T$}
		\State{{\bf Initialize:} \newline $V^{(1)},W^{(1)},b^{(1)}{\gets}$ Optimize \eqref{eq:training_obj_2} over $V,W,b$ using $X_\mcs$ and $Y_\mcs$, starting from $V^{(0)}$}
		\For{$t = 1, \ldots, T$}
		\State{$X^{(t)}, Y^{(t)} \gets \text{Draw minibatch of samples}$ }
		\State{$M^{(t)} \gets \text{MatrixBalancing}(A_\regtwo^{(t)}, Y^{(t)}{Y^{(t)}}^T) $}
		\State{$V^{(t+1)} \gets \text{ULR-SGO step}(\Phi_{V^{(t)}}(X^{(t)}), M^{(t)}, V^{(t)})$}
		\EndFor
		\State{$\hat Y_\mcu \gets \text{NearestNeighbor}(\Phi_{V^{(T+1)}}(X),Y_\mcs)$ }
		\State{$\hat W, \hat b \gets \text{RegLeastSquares}(X, [Y_\mcs, \hat Y_\mcu]) $}
		\vspace{0.12cm}
		\State{{\bf Output: } $\hat Y_\mcu, V^{(T+1)}, \hat W, \hat b$}	
	\end{algorithmic}
\end{algorithm}
 \begin{algorithm}[H]
	\caption{\label{alg:xsdc_unsup} XSDC (when no labeled data exists)}
	\begin{algorithmic}[1]
		\State{{\bf Input:} %
		Unlabeled data $X_\mcu$ \\
		\hspace{1.0cm} Randomly initialized network parameters $V^{(1)}$\\
		\hspace{1.0cm} Number of iterations $T$}
		\For{$t = 1, \ldots, T$}
		\State{$X^{(t)} \gets \text{Draw minibatch of samples}$ }
		\State{$M^{(t)} \gets \text{MatrixBalancing}(A_\regtwo^{(t)}, \Id_{\batchsize}) $}
		\State{$V^{(t+1)} \gets \text{ULR-SGO step}(\Phi_{V^{(t)}}(X^{(t)}), M^{(t)}, V^{(t)})$}
		\EndFor
		\State{$M^{(T+1)} \gets \text{MatrixBalancing}(A_\regtwo^{(T+1)}(\Phi_{V^{(T+1)}}(X_{\mcu})), \Id_{n_\mcu}) $}
		\State{$\hat {Y}_\mcu \gets \text{SpectralClustering}(M^{(T+1)})$ }
		\State{$\hat W, \hat b \gets \text{RegLeastSquares}(X; \hat Y_\mcu) $}
		\State{{\bf Output: } $\hat Y_\mcu, V^{(T+1)}, \hat W, \hat b$}	
	\end{algorithmic}
\end{algorithm}

 \paragraph{XSDC algorithm.}
The overall XSDC algorithm for the case where some labeled data is present is summarized in Algorithm~\ref{alg:xsdc}. The algorithm proceeds as follows. First, we initialize the parameters $V$ randomly and then optimize the objective on the labeled data to obtain initial estimates of $V,W,$ and $b$. Next, we proceed to optimize using the labeled and unlabeled data together. At each iteration, we draw a mini-batch of $\batchsize$ inputs $X^{(t)}= (x^{(t)}_{1}, \ldots, x^{(t)}_{\batchsize})$ with corresponding labels $Y^{(t)}$ (some known, some unknown). We compute the network output $\Phi_{V^{(t)}}(X^{(t)}) = (\phi_{V^{(t)}}(x^{(t)}_{1}) ,\allowbreak\ldots, \phi_{V^{(t)}}(x^{(t)}_{\batchsize}) )^T $ , followed by $A^{(t)}_\regtwo(\Phi_{V^{(t)}}(X^{(t)}))=\\\regtwo\Pi_{\batchsize}(\Pi_{\batchsize}\Phi_{V^{(t)}}(X^{(t)})\Phi_{V^{(t)}}(X^{(t)})^T\Pi_{\batchsize} + \batchsize\regtwo \id_{\batchsize})^{-1}\Pi_{\batchsize}$.

We then perform matrix balancing to obtain $M^{(t)}$. Fixing $M^{(t)}$, we then take a gradient step based on the ULR-SGO objective. 
Once the feature representation has been optimized, we obtain labels $ \hat Y_\mcu$ for the unlabeled data using 1-nearest neighbor on the feature representations $\Phi_{V^{(T+1)}}(X)$.
 Finally, we estimate the parameters $W$ and $b$ by computing the solution to the least squares problem with $X$ and $[Y_\mcs, \hat Y_\mcu]$.

The algorithm in the special case in which there is no labeled data is summarized in Algorithm~\ref{alg:xsdc_unsup}. Aside from removing the supervised initialization step, the only difference lies in the estimation of $\hat Y_\mcu$ and the evaluation of the performance. Specifically, since we do not have any labeled data with which to perform nearest neighbor estimation, we instead use spectral clustering. Note that the cluster numbers output by spectral clustering do not necessarily map to the correct labels (e.g., cluster 0 might correspond to the label 1 rather than 0). Therefore, to evaluate the accuracy of the method we find the optimal relabeling of the classes that aligns with the true labels. We do so by solving a maximum weight matching problem with the Hungarian algorithm~\citep{schrijver2003}. In this special case, the algorithm allows one to equip the DIFFRAC method with a representation learning ability, extending the original work of~\cite{bach2007}

The XSDC algorithm has two benefits. First, learning the features does not require knowledge of the number of clusters. Instead, it requires only a bound on the fraction of points per cluster, for use in the matrix balancing. Specifying such a bound is easier than providing the number of clusters. The only time we must use knowledge of the number of clusters is when evaluating the performance of the learned features. Second, the algorithm is extendable to the case where we have additional must-link or must-not-link information related to the labels. For example, if we know observations $i$ and $j$ must not have the same label, we can encode that constraint in the above problem by adding $(i,j)$ to $\mck$ and setting $m_{ij}=0$. The algorithm itself is otherwise identical.  This is an important extension for cases where the sources of annotation (such as human annotators from crowdsourcing platforms) may have been unsure and failed to produce a label for an observation (e.g., ``Welsh springer spaniel'') but could provide certain relevant label information (e.g., the dog is not the same breed as the dog in another image).

\section{Experiments}
\label{sec:experiments}

In the experiments we illustrate how the proposed approach can be used to perform discriminative clustering while learning a feature representation and leveraging any amount of labeled data at hand. 
Exploring specialized versions of our algorithm for specific applications is beyond the scope of this paper. 
We focus on unifying learning with no labeled data, some labeled data, and fully-labeled data in a single training objective. 
Recall that the proposed algorithm is referred to as XSDC in the tables and in the figures. 
 \begin{table*}[t]
 \centering
  \caption{\label{tab:dataset_info} Details regarding the datasets used in the experiments.}
\begin{tabular}{c|rrrrr}
Dataset  & Training size  &  Validation size & Test size & Dimension & \# Classes\\
\toprule
CIFAR-10 &40,000 & 10,000  &10,000& 3,072 & 10\\
Gisette & 4,800 & 1,200 & 1,000 & 5,000 & 2\\
MAGIC & 8,026 & 2,006  & 3,344 & 10 & 2\\
MNIST & 50,000 & 10,000 &10,000 & 784 & 10\\
\bottomrule
\end{tabular}
\end{table*}

\subsection{Choice of $\phi_V$}
One benefit of the XSDC algorithm is that it can actually learn a similarity measure for similarity-based clustering. Typical clustering methods either do not transform the features or use a kernel-based method. However, clustering in the original space of raw features can be ineffective in many problems. Moreover, clustering using the Gram matrix on the inputs is computationally infeasible when there are a large number of observations and may fail when the kernel is improperly chosen~\citep{perez-cruz2004}. 

We use kernel networks to define the feature representation mapping $\phi_V$. Several methods for approximating kernels exist, including random Fourier features and the Nystr\"{o}m method \citep{rahimi2007,williams2000,mohri2012,daniely2016}. Random Fourier features are data-independent and the parameters of the Nystr\"{o}m method are typically selected at random or via a quantization procedure~\citep{oglic2017}. 

We instead learn the parameters of Nystr\"{o}m approximations of kernels at each layer, similarly to~\citet{mairal2016}.  Following~\citet{mairal2016}, the regularized Nystr\"{o}m approximation method approximates a kernel $k$ by computing the inner products of features $\phi(x)$ defined by $\phi(x)=(k(V^TV)+\epsilon I)^{-1/2}k(V^Tx)$ for some small $\epsilon>0$ where the matrix $V$ contains the parameters and $k$ is understood to be applied element-wise. 

We expect similar behavior for other kinds of networks, given observations made by \citet{lee2018,matthews2018}, and \citet{belkin2018}.

\subsection{Experimental details}
\label{sec:experiments_details}

\paragraph{Experimental setup.} The experiments focus on four datasets: the vectorial datasets Gisette~\citep{guyon2004} and MAGIC \citep{bock2004}, as well as the image datasets MNIST~\citep{lecun2001} and CIFAR-10~\citep{krizhevsky2009}. Gisette contains features derived from images of digits, and the goal is to distinguish between observations corresponding to the numbers four and nine. MAGIC contains measurements related to simulated particles observed by a gamma telescope. The aim is to distinguish between gamma particles and hadrons. MNIST contains images of the digits 0-9 and the objective is to be able to distinguish between all ten digits. Finally, CIFAR-10 contains images of ten different objects (e.g., birds, planes), and the goal is to classify the object in each image.

The details of the sizes and dimensions of each dataset we consider can be found in Table~\ref{tab:dataset_info}. 
For the MAGIC dataset, which does not have a train/test split, we randomly split the data 75\%/25\% into train/test sets.  
For Gisette, MAGIC, and CIFAR-10 we set aside 20\% of the training set to use as a validation set, while for MNIST we set aside the standard 17\%. In one experiment we vary the distribution of the labels in MNIST. 
For this experiment we use 25,000 unlabeled images, 50 of which are labeled. Each class with labels 0-4 has the same number of unlabeled observations (e.g., 3992 per class when labels 0-4 make up 80\% of the data), and same for classes 5-9 (998 per class when labels 5-9 make up 20\% of the data). However, the labeled data is still balanced.

The datasets are transformed prior to usage as follows. Gisette is the scaled version found in the LibSVM database~\citep{chang2011}. MAGIC and MNIST are standardized. For CIFAR-10, we use the gradient map. %
As some of our experiments use the version of XSDC that assumes the classes are balanced, we randomly remove from the MAGIC dataset 5,644 observations in the dataset with label 1.

The architectures we use in the experiments are kernel networks. For the vectorial datasets we use single-layer kernel networks (KNs) that approximate a Gaussian RBF kernel using the Nystr\"{o}m method. In contrast, for MNIST we use a convolutional kernel network (CKN) translation of LeNet-5~\citep{lecun2001} and for CIFAR-10 we use a CKN applied to the gradient map on the inputs (CKN-GM)~\citep{mairal2014}.
For each of these networks we use 32 filters per layer for the hidden layers. These architectures and datasets were chosen because they represent a broad spectrum in terms of performance. 
 For details on the hyperparameter values and hold-out validation, see Appendix~\ref{app:training_details}. The hold-out validation is performed on the datasets for each quantity of labeled data but with a single random seed. The best parameters found are used for all other random seeds. 

\begin{figure}[t!]
	\centering
	\includegraphics[width=1.0\linewidth,trim={14cm 0cm 7cm 0cm},clip]{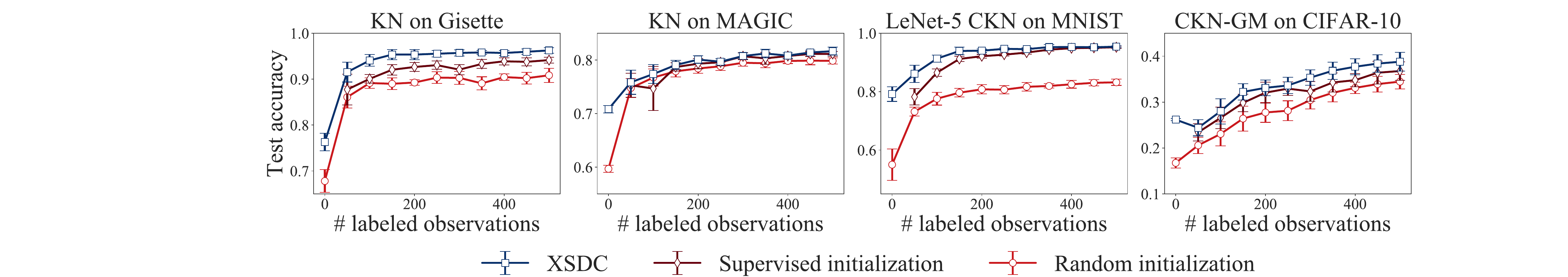}
	\caption{\label{fig:acc_vs_num_labeled} Average performance across 10 trials of XSDC when varying the quantity of labeled data. The error bars show one standard deviation from the mean.}
\end{figure}

\paragraph{Training.} The training is performed as follows. 
The network parameters are initialized by randomly sampling from the feature representations at each layer of the network. 
Then the network is trained for 100 iterations using the labeled data. 
Finally, the network is trained on the labeled and unlabeled data for 400 iterations, using matrix balancing to predict the labels of the unlabeled data. Unless otherwise specified, $n_{\min}=n_{\max}$ in the matrix balancing, i.e., all classes are assumed to be equally represented within each mini-batch. We evaluate the performance of the learned representations every 10 iterations.

\paragraph{Code.} The code for this project is written using Faiss, PyTorch, SciPy, and YesWeCKN~\citep{johnson2019,paszke2019, virtanen2020,jones2020}. It can be found online at \url{https://github.com/cjones6/xsdc}.

\subsection{Results}

\paragraph{Improvement with unlabeled data.} In the experiments we first compare the XSDC algorithm to two simple baselines: an initial supervised training of the classifier when the network has random weights (``random initialization'') and an initial supervised training of both the network and the classifier (``supervised initialization''). In the latter case the network is trained on only the labeled data. In both cases, when evaluating the performance the labels of the unlabeled data are first estimated using 1-nearest neighbor with the labeled data based on the learned features. The classifier is then trained on the labeled and unlabeled data. The reported accuracy of the supervised initialization is the test accuracy after 100 iterations. In contrast, the reported accuracy of XSDC when labeled data exists is the test accuracy observed at the iteration where the validation accuracy is highest. We report this value because the algorithm can overfit before 500 iterations. In the case where no labeled data exists we report the highest observed test accuracy. We performed 10 trials when varying the random seed and report the mean and standard deviation of the corresponding results.

We would expect that XSDC would provide an improvement over the supervised initialization when there are gains to be had from additional labeled data. Otherwise, we would expect training on additional unlabeled data to provide little to no benefit. This is what we see in Figure~\ref{fig:acc_vs_num_labeled}. 
Figure~\ref{fig:acc_vs_num_labeled} compares the accuracy of the XSDC algorithm to the initializations as the quantity of labeled data varies. From all of the plots we can see that the performance of XSDC relative to the supervised baseline is much larger when the quantity of labeled data is smaller. With 50 labeled examples the accuracy on Gisette increases by 4\% on average when using XSDC instead of the supervised baseline. On MAGIC the gain is more modest, at 0.8\%. For MNIST the gain is 10\%, while for CIFAR-10 it is 4\%. In contrast, for 500 labeled observations XSDC outperforms the supervised baseline by 2\% on Gisette but is only 0.6\% better than the baseline on MAGIC. The latter results make sense since the increase in performance of the supervised initialization with the quantity of labeled data has started leveling off by then. On MNIST the improvement when there are 500 labeled observations drops to 0.2\%, while on CIFAR-10 it is 6\%. Note that the drop in accuracy of XSDC on CIFAR-10 from zero to 50 labeled observations is likely because we report the highest observed test accuracy for the case of zero labeled observations. 

\def\tsnewidth{1.0}
\begin{figure}
\begin{subfigure}{.24\linewidth}
\includegraphics[width=\tsnewidth\columnwidth]{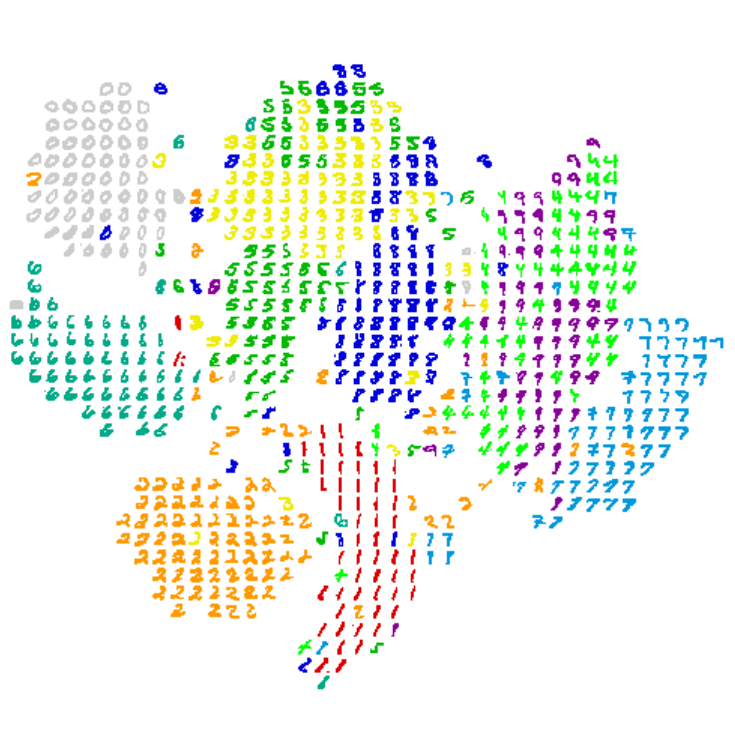}
\caption{\label{fig:tsne_raw}Raw}
\end{subfigure}
\begin{subfigure}{.24\linewidth}
\includegraphics[width=\tsnewidth\columnwidth]{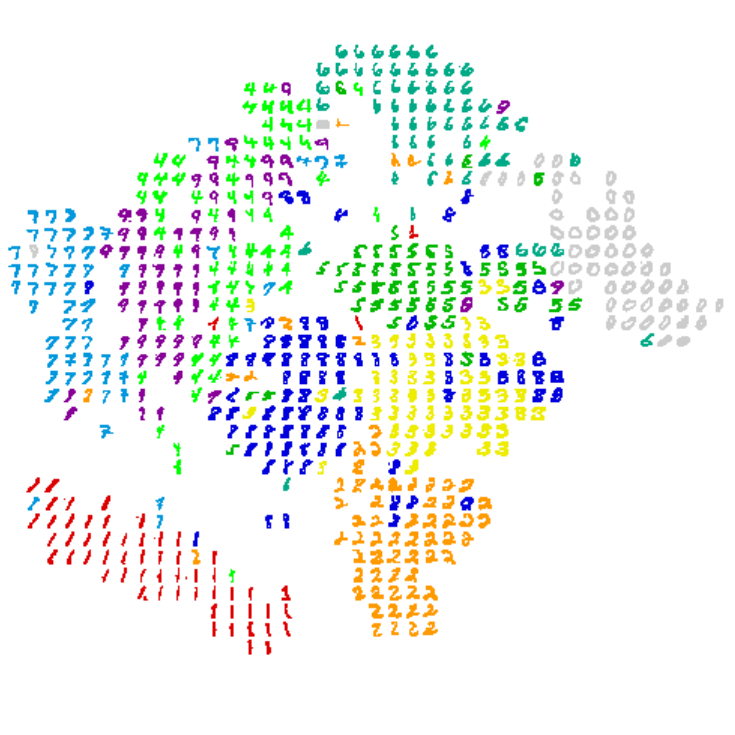}
\caption{\label{fig:tsne_unsup} Unsupervised init.}
\end{subfigure}
\begin{subfigure}{.24\linewidth}
\includegraphics[width=\tsnewidth\columnwidth]{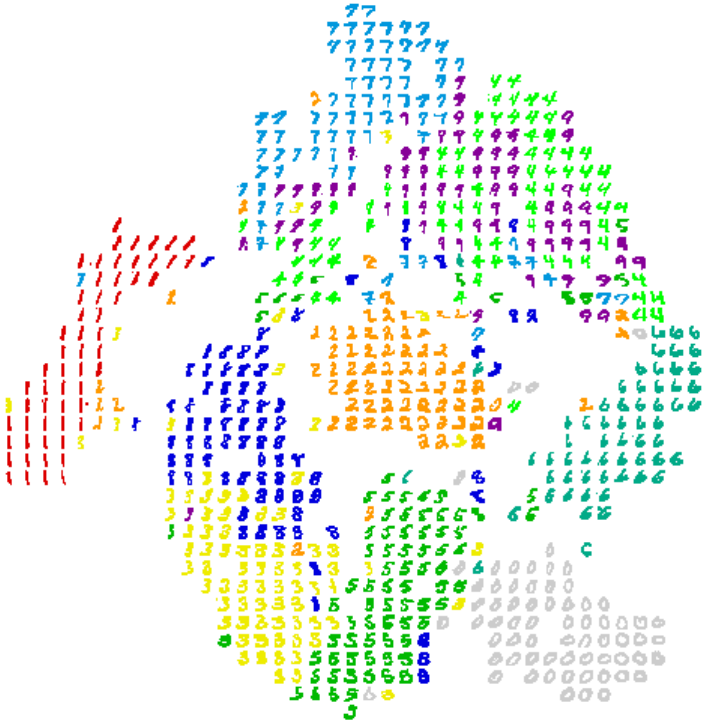}
\caption{\label{fig:tsne_sup} Supervised init.}
\end{subfigure}\hfill
\begin{subfigure}{.24\linewidth}
\includegraphics[width=\tsnewidth\columnwidth]{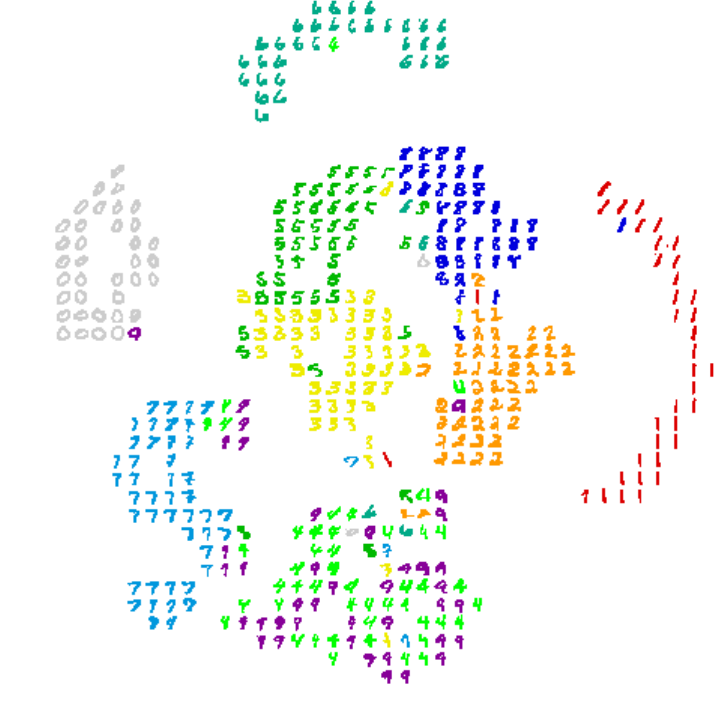}
\caption{\label{fig:tsne_xsdc} XSDC}
\end{subfigure}
\caption{\label{fig:mnist50_viz} Visualizations of the unlabeled MNIST features obtained when training the LeNet-5 CKN with 50 labeled observations (where applicable). The CKN features were projected to 2-D using t-SNE. The features were obtained at different stages, as indicated in the sub-captions.}
\end{figure}

\begin{figure}
	\includegraphics[width=\linewidth]{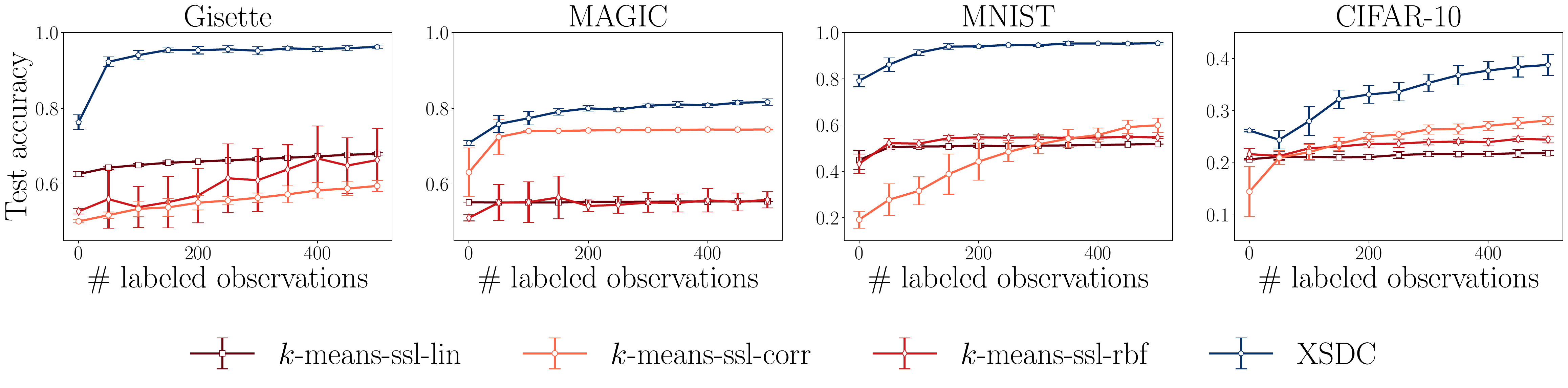}
	\caption{Average performance of semi-supervised $k$-means with various fixed metrics (from left to right: with a linear similarity measure, with a similarity measure defined by the inverse correlation matrix, with a non-linear similarity measure defined by an RBF kernel) compared to XSDC, when varying the quantity of labeled data. The error bars show one standard deviation from the mean.\label{fig:baseline}}
\end{figure}

There are two other noteworthy aspects of Figure~\ref{fig:acc_vs_num_labeled}. First, it shows that XSDC can improve over the unsupervised initialization even in the case where there is no labeled data. The relative improvement in accuracy over the unsupervised baseline ranges from 13\% on Gisette to 56\% on CIFAR-10 when no labeled data is present. Second, the standard deviation of the difference in the performance between the supervised baseline and XSDC tends to be larger when the gap in the performance between XSDC and the supervised baseline is larger, as expected. For example, on Gisette the standard deviation of the difference in the performance of XSDC and the supervised baseline is 4\% in the case of 50 labeled observations, but only 0.3\% in the case of 500 labeled observations.

We also visualize the results, examining the case where 50 images from MNIST are labeled.  Figure~\ref{fig:mnist50_viz} depicts the feature representations of a batch of 4096 unlabeled observations at various points of the training process. For each plot the feature representations were projected to 2-D using t-SNE \citep{vandermaaten2008}. For each square in a grid, the code checks whether any image's t-SNE representation lies in that square. If any such images exist, it chooses one at random and displays the original image in that square. The images are then color-coded according to the ground-truth labels.\footnote{The code to produce the plots was adapted from Andrej Karpathy's Matlab code, which can be found here: \url{https://cs.stanford.edu/people/karpathy/cnnembed/}.} Comparing Figures~\ref{fig:tsne_sup}  and \ref{fig:tsne_xsdc}, we can see that XSDC tends to increase the separation between clusters relative to the supervised initialization. This suggests that the inter-class distances between the feature representations learned by XSDC tend to be larger relative to the intra-class distances. The digits 4, 7, and 9 are a bit less separated. However, the digits 5 and 8 are each generally all in one cluster after running XSDC. 

\paragraph{Comparison to semi-supervised learning methods with a fixed representation.}
Our goal is to provide an algorithm that can both (i) learn a good representation of the data; and (ii) take advantage of all available data (labeled and unlabeled). To understand the benefits of learning a feature representation, we consider as a baseline semi-supervised $k$-means with seeding. This is a popular variant of $k$-means where the centroids are initialized by the labeled data and the assignments of the labeled data are fixed to their given labels~\citep{basu2002semi, yoder2017semi}. 
Once the clusters are found in the training data, we predict the label of new data points  by assigning them to the closest cluster. If no data points are present, we use $k$-means with $k$-means++ initialization.

The fixed similarity measures defining the clusters in our implementations of $k$-means are
(i) a linear similarity measure: $h(x, y) = x^\top y$,  which amounts to clustering points with respect to the squared Euclidean distances in the original feature space; (ii) a data-dependent similarity measure defined by the regularized inverse correlation matrix: $h(x, y) = x^\top (X^\top X/(n-1) +\lambda I)^{-1} y$, where $X= (x_1, \ldots, x_n)^\top \in \mathbb{R}^{n \times d}$ is the set of all standardized training points and $\lambda \geq 0$ is a regularization parameter; DIFFRAC~\citep{bach2007} implicitly uses an analogous similarity measure; and (iii) a non-linear similarity measure defined by a Gaussian Radial Basis Function (RBF) kernel: $h(x, y) = \exp\left(- \|x-y\|_2^2/(2\sigma^2)\right)$ for some bandwidth parameter $\sigma>0$, which amounts to clustering points in the reproducing kernel Hilbert space associated with $h$.

The bandwidth parameter of the RBF kernel is chosen using the median heuristic. Namely, we choose $\sigma = \sqrt{2 d}$, where $d = \textrm{Median}((\|x_i-x_j\|_2^2)_{1\leq i<j\leq n})$ is the median squared distance computed from all data (unlabeled and labeled ones)~\citep{fukumizu2009kernel}. The regularization parameter for the similarity measure based on the correlation matrix is chosen using the heuristic $\lambda = \lambda_{\max}(C)/\mathrm{Tr}(C)$, where $C = X^\top X/(n-1)$.

\begin{figure}[t]
	\centering
	\includegraphics[width=1.0\linewidth,trim={14cm 0cm 7cm 0cm},clip]{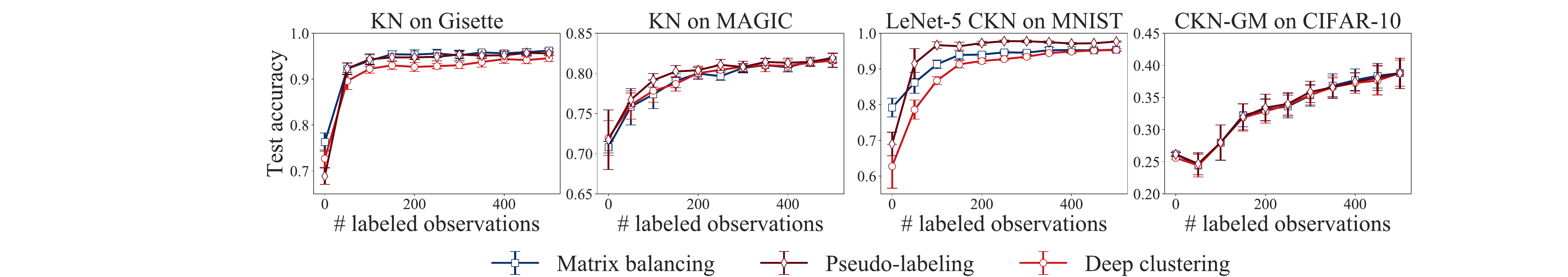}
	\caption{\label{fig:acc_vs_num_labeled_comparison} Average performance across 10 trials of XSDC with matrix balancing and two alternative labeling methods (pseudo-labeling and deep clustering) when varying the quantity of labeled data. The error bars show one standard deviation from the mean.}
\end{figure}

Our experimental setting is the same as before: we fix the total amount of data for each dataset (Gisette, MAGIC, MNIST, CIFAR) and vary the amount of labeled data. In all cases, we standardized the data before applying the algorithm.  We run the algorithm 10 times with a different set of labeled points each time or a different seed for the $k$-means++ initialization. We compare semi-supervised $k$-means to XSDC with the same architectures as the ones presented in Figure~\ref{fig:acc_vs_num_labeled_comparison}. In Figure~\ref{fig:baseline}, we observe an increase of performance of the semi-supervised $k$-means algorithm as the number of labels increase. However, in general, learning a feature representation as done by XSDC leads to better accuracy results even with a small amount of labeled data points.

\paragraph{Comparison to alternative labeling methods.} Next, we compare to two alternative labeling methods: pseudo-labeling \citep{lee2013} and deep clustering \citep{caron2018}.
Pseudo-labeling is a method designed to learn feature representations from labeled data and unlabeled data. Label assignment is performed by predicting labels from regression on the learned features. 
In contrast, deep clustering is a method designed to learn feature representations from unlabeled data and assign labels to unlabeled data. Label assignment is performed by $k$-means clustering with the learned features. Designing a variant working with both labeled data and unlabeled data was beyond the scope of \citet{caron2018}. 
See Appendix~\ref{app:training_competitors} for how we adapted pseudo-labeling and deep clustering to the unsupervised setting and the semi-supervised setting, respectively.

Figure~\ref{fig:acc_vs_num_labeled_comparison} displays results comparing the labeling method in XSDC (matrix balancing) to the labeling methods from pseudo-labeling and deep clustering. From the plots, we can see that the accuracy with matrix balancing and pseudo-labeling are only significantly different when training the LeNet-5 CKN on MNIST. However, both matrix balancing and pseudo-labeling typically outperform deep clustering when training the kernel network on Gisette and the LeNet-5 CKN on MNIST. On average, matrix balancing is 1-5\% better than deep clustering when training the kernel network on Gisette and 0.1-28\% better than deep clustering when training the LeNet-5 CKN on MNIST. These results suggest that for certain  architectures and datasets, using label information may be essential to achieving a performance close to the best possible one. On the other hand, the choice of how that label information is incorporated, whether it is by matrix balancing or pseudo-labeling, may matter less frequently in terms of the performance.

\begin{figure}[t]
	\begin{minipage}{0.48\linewidth}
	\centering
	\includegraphics[width=\linewidth]{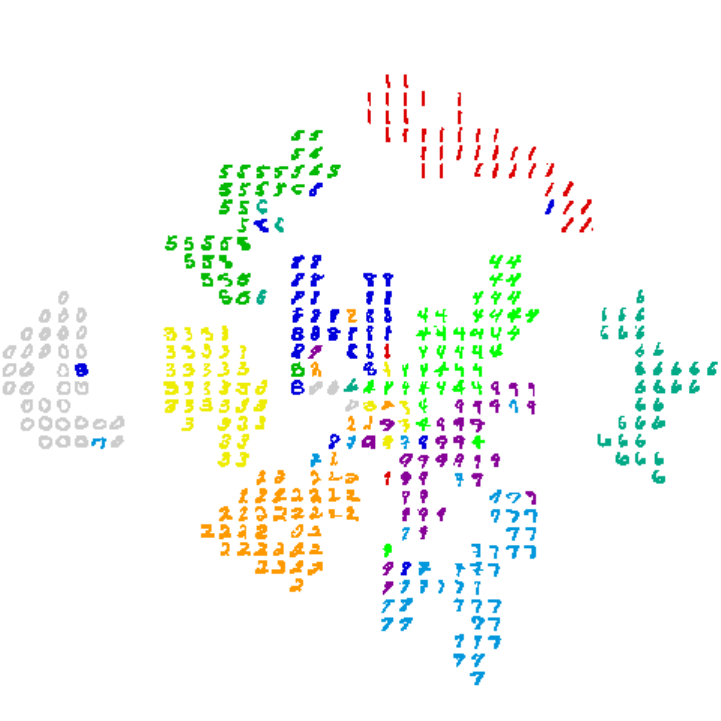}
	\caption{\label{fig:mnist50_500_viz_addl} Visualization of the unlabeled MNIST features obtained when training the LeNet-5 CKN with 50 labeled observations and additional known constraints. The CKN features were projected to 2-D using t-SNE. The constraints were derived from knowledge of whether the label for each unlabeled point lies in the set \{4,9\}.}
	\end{minipage}\hspace{5pt}
	\begin{minipage}{0.48\linewidth}
	\centering
	\vspace{30pt}
	\includegraphics[width=\linewidth]{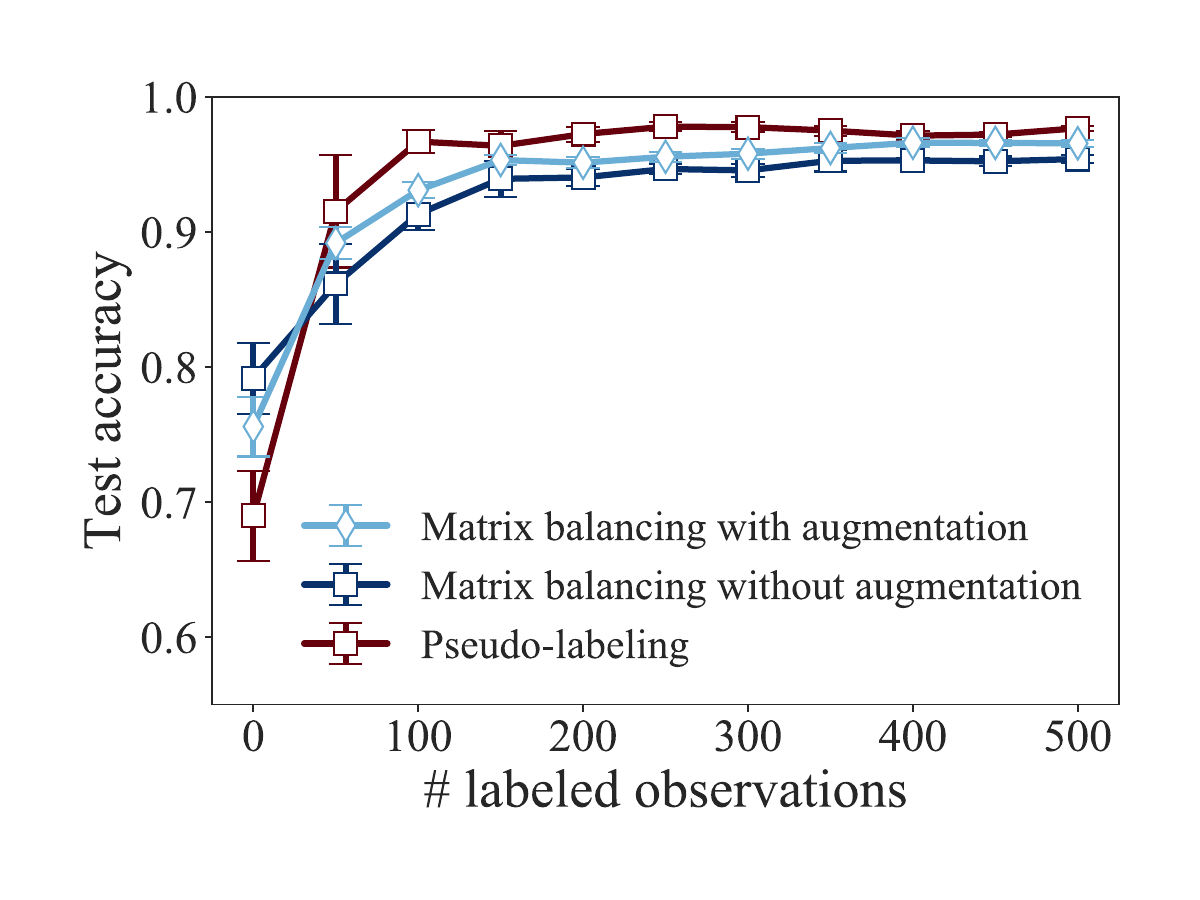}
	\caption{\label{fig:aug} Average performance across 10 trials of XSDC with matrix balancing with and without data augmentation and of XSDC with pseudo-labeling. The error bars show one standard deviation from the mean.}
	\end{minipage}
\end{figure}

\paragraph{Improvement with domain-specific knowledge.}
Recall that XSDC with matrix balancing is able to enforce must-link constraints regarding the labels, via the first constraint from problem~\eqref{eq:balancing_obj}. We can make use of this by performing data augmentation, i.e., generating copies of observations in the dataset with slightly modified features. We tried this on MNIST, using random rotations, random image widths, random shifts, and random erasures, as done by \citet{byerly2020}. The number of augmentations per batch was chosen from $2^i$, $i=0, 1, 2, \dots, 7$ via hold-out validation. Because of the increase in the number of images due to augmentations, the batch size was modified accordingly. When using XSDC we enforced constraints requiring the augmentations of the same image to have the same label. The results are shown in Figure~\ref{fig:aug}. We can see that this approach partly closes the observed gap between matrix balancing and pseudo-labeling that was observed in Figure~\ref{fig:acc_vs_num_labeled_comparison}.

\paragraph{Improvement with additional constraints.}
As noted in Section~\ref{sec:opt}, XSDC can seamlessly incorporate additional must-link and must-not-link constraints. To assess the benefit of adding such constraints, we provide additional experiments with the LeNet-5 CKN on MNIST. We consider two forms of additional constraints: (a) must-not-link constraints derived from knowledge of whether or not each unlabeled observation was from either class 4 or 9; and (b) random correct must-link and must-not-link constraints among pairs of unlabeled observations and random correct must-not-link constraints between pairs of unlabeled and labeled observations. The pairs of classes in (a) were selected because they are frequently confused. This attempts to mimic a situation in which a labeler knows that an observation belongs to one of two classes, but is not sure which one. Each random constraint in (b) was added with probability 1/3, yielding approximately the same number of constraints as (a). See Appendix~\ref{sec:addl_constraints} for additional details.

Figure~\ref{fig:mnist50_500_viz_addl} visualizes the feature representations resulting from constraints of the form (a) for the case of 50 labeled observations from MNIST.  Examining this figure, we can see that the clusters are generally well-separated, including the bright green, light blue, and purple clusters, which correspond to the digits 4, 7, and 9, respectively. Visually, this is an improvement over the t-SNE projections when the additional constraints are not used (\emph{cf.} Figure~\ref{fig:tsne_xsdc}).

Next, Figure~\ref{fig:addl_constraints} displays results comparing the test accuracy on MNIST when including and not including the additional constraints on the labels. As expected, adding the additional constraints generally improves the performance. The addition of random correct constraints results in the best performance, likely because these provide more knowledge related to the difficult-to-distinguish classes.

\begin{figure}[t]
\begin{minipage}{0.48\linewidth}
\centering
\includegraphics[width=\linewidth]{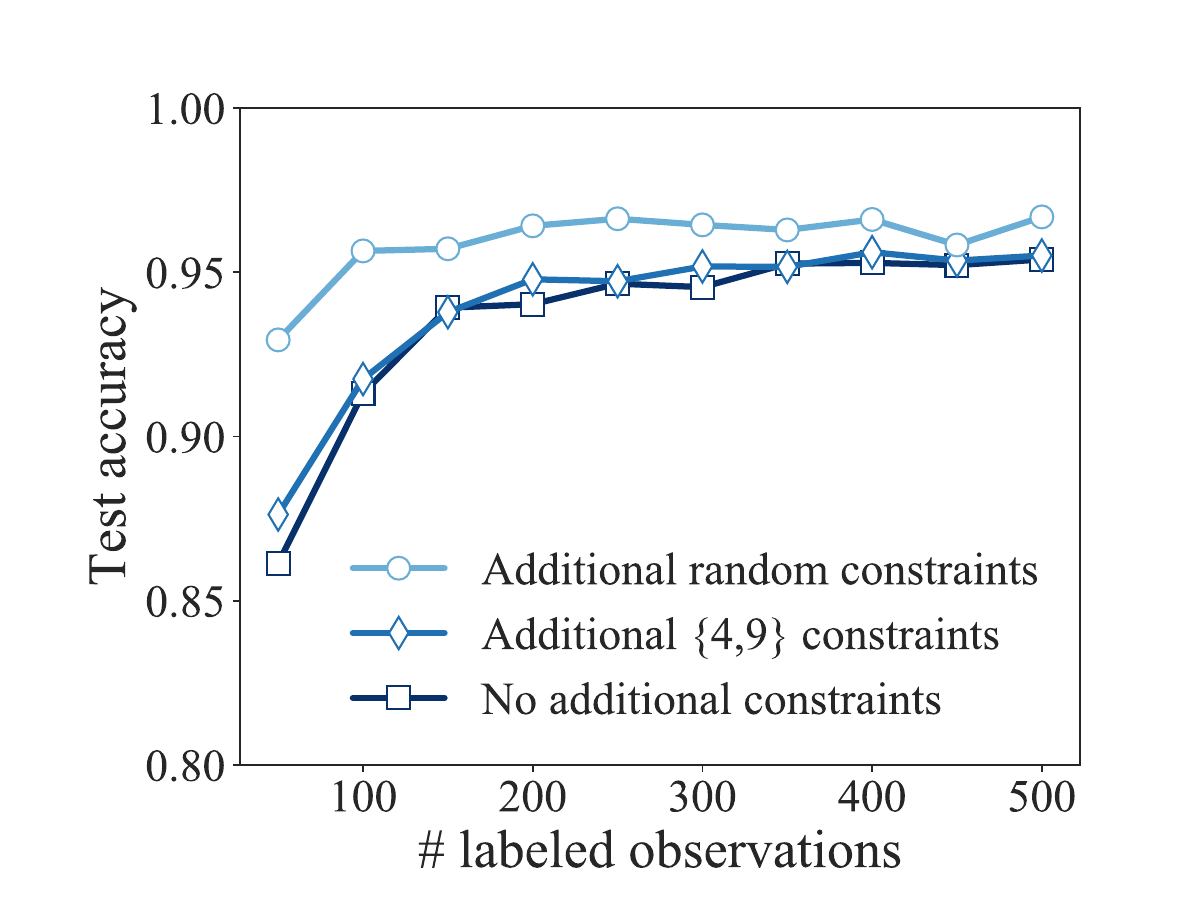}
\caption{\label{fig:addl_constraints} Average accuracy across 10 trials of XSDC after training a LeNet-5 CKN on MNIST when adding additional constraints.}
\end{minipage}~\hspace{5pt}
\begin{minipage}{0.48\linewidth}
	\centering
	\vspace{47pt}
\includegraphics[width=\linewidth]{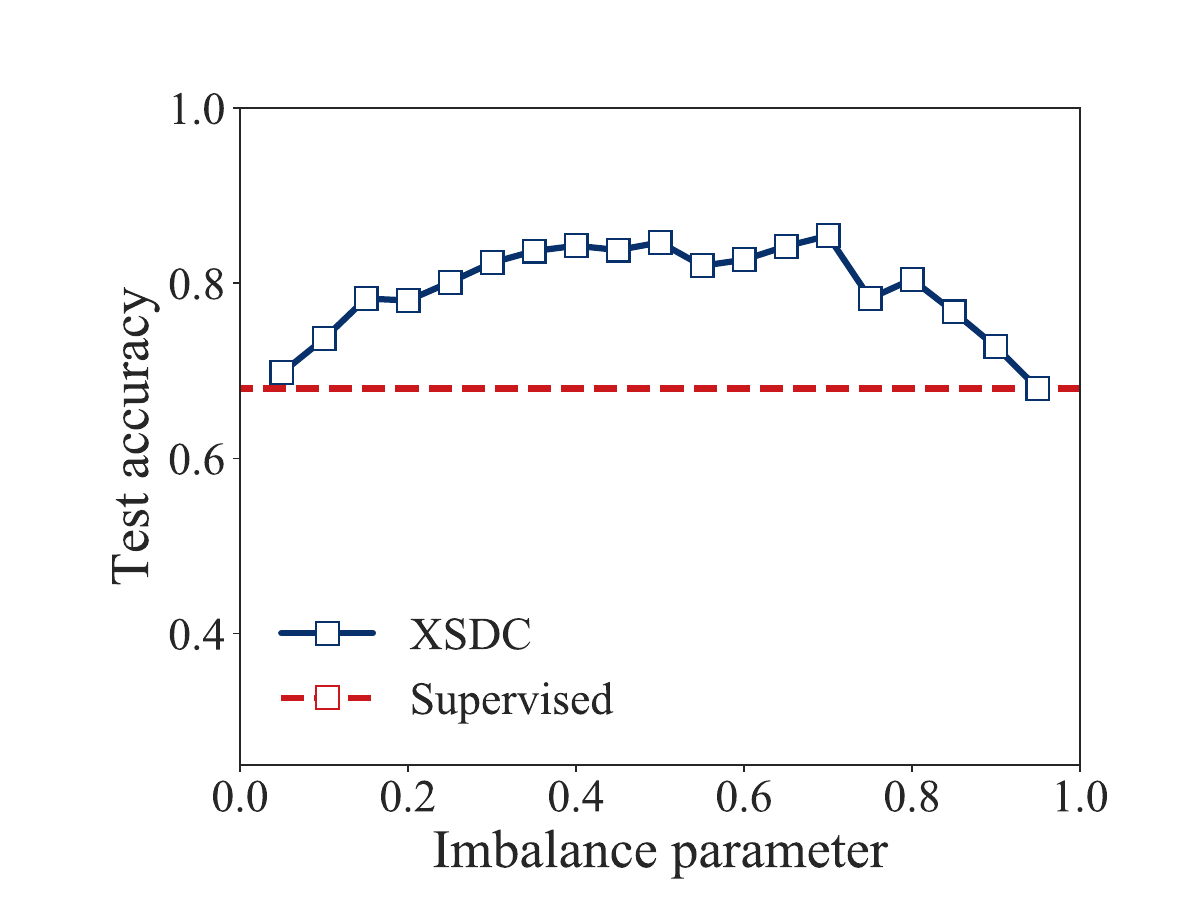}
\caption{\label{fig:imbalanced data_mnist} Average accuracy across 10 trials of XSDC after training a LeNet-5 CKN on MNIST when varying the fraction of labeled data. The imbalance parameter denotes the fraction of the labels that are from the set $\{0,1,2,3,4\}$. All classes in the set $\{0,1,2,3,4\}$ are equally represented, and similarly for $\{5,6,7,8,9\}$.}
\end{minipage}
\end{figure}

\paragraph{Performance with unbalanced data.}
The XSDC algorithm can handle unbalanced datasets by changing the bounds on the cluster sizes in the matrix balancing algorithm. To present an example of how XSDC performs on unbalanced unlabeled data we again trained the LeNet-5 CKN on MNIST. We used 50 labeled observations, equally distributed across classes. For the unlabeled data we varied the fraction of labels 0-4 and the fraction of labels 5-9 between 5\% and 95\%. For training we use the hold-out validation set to determine the bounds on the cluster sizes.

The results are presented in Figure~\ref{fig:imbalanced data_mnist}. Training with XSDC on both the labeled and unlabeled data is nearly always better than training on the labeled data only (dashed curve). As expected, the performance tends to be better for more balanced data. The best accuracy was 85\%, obtained with 70\% 0-4's, while the worst accuracy was 68\%, obtained with 95\% 0-4's. In contrast, the accuracy when training on only the labeled data was 68\%. These results suggest that as long as one believes that the unlabeled data is not extremely unbalanced, it could be beneficial to use it during training.

\begin{figure}[t]
\begin{subfigure}{.33\linewidth}
\includegraphics[width=1.0\linewidth]{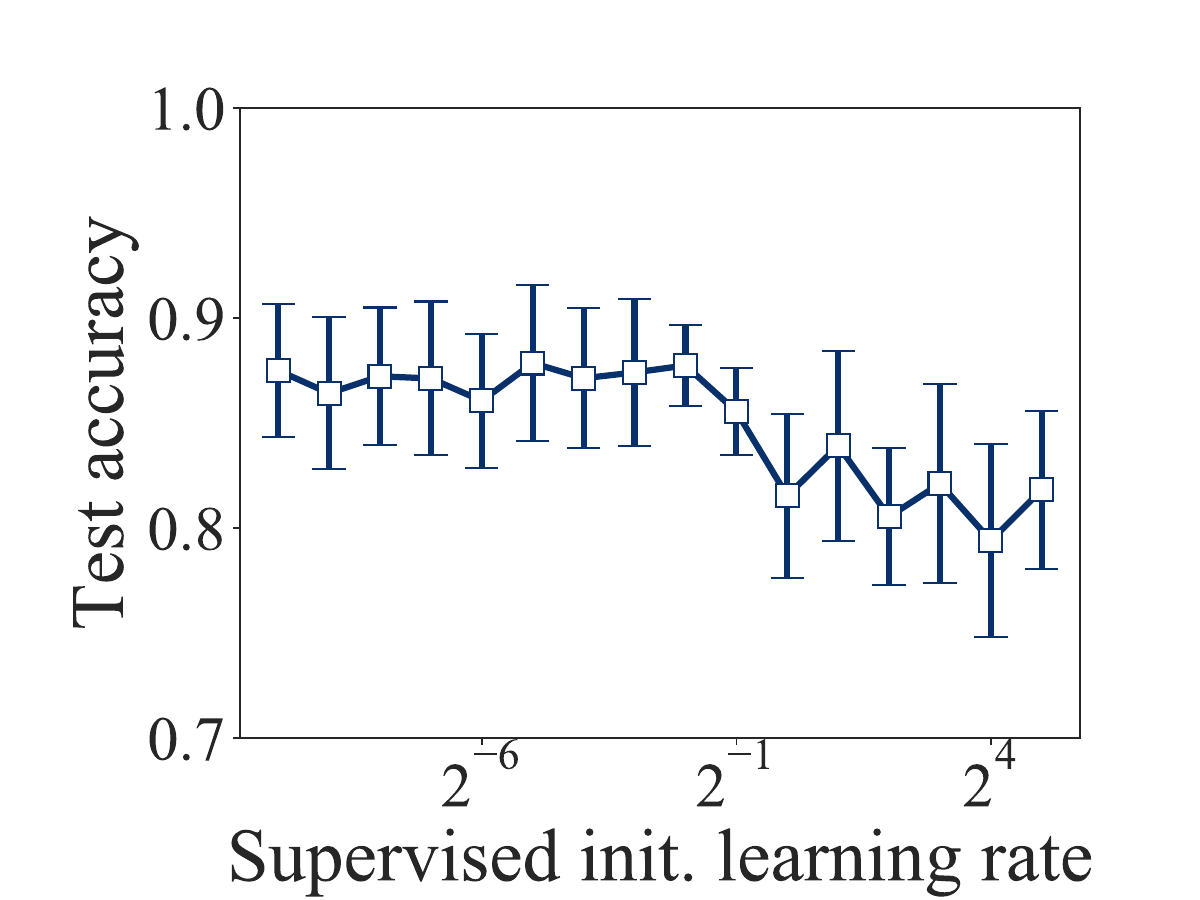}
\end{subfigure}\hfill
\begin{subfigure}{.33\linewidth}
\includegraphics[width=1.0\linewidth]{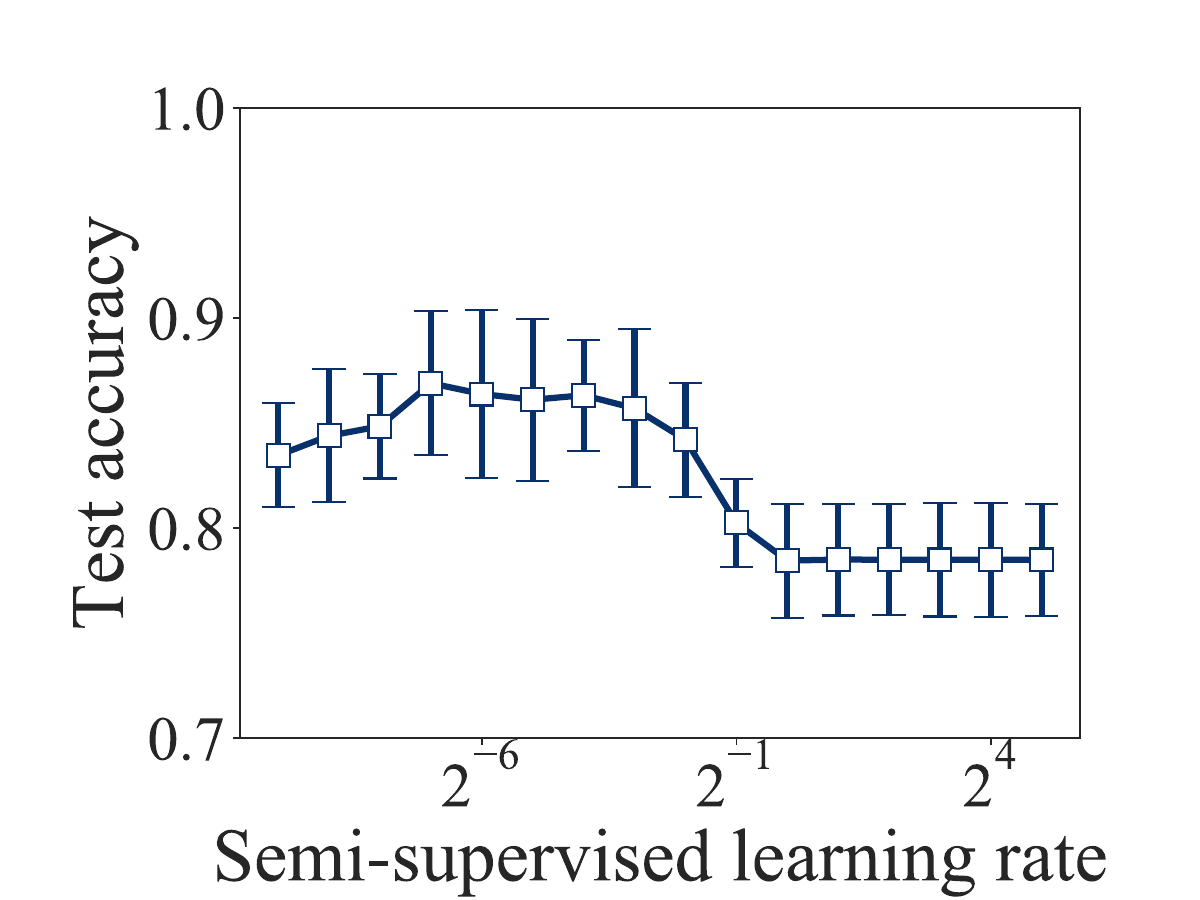}
\end{subfigure}\hfill
\begin{subfigure}{0.33\linewidth}
\begin{center}
\includegraphics[width=1.0\linewidth]{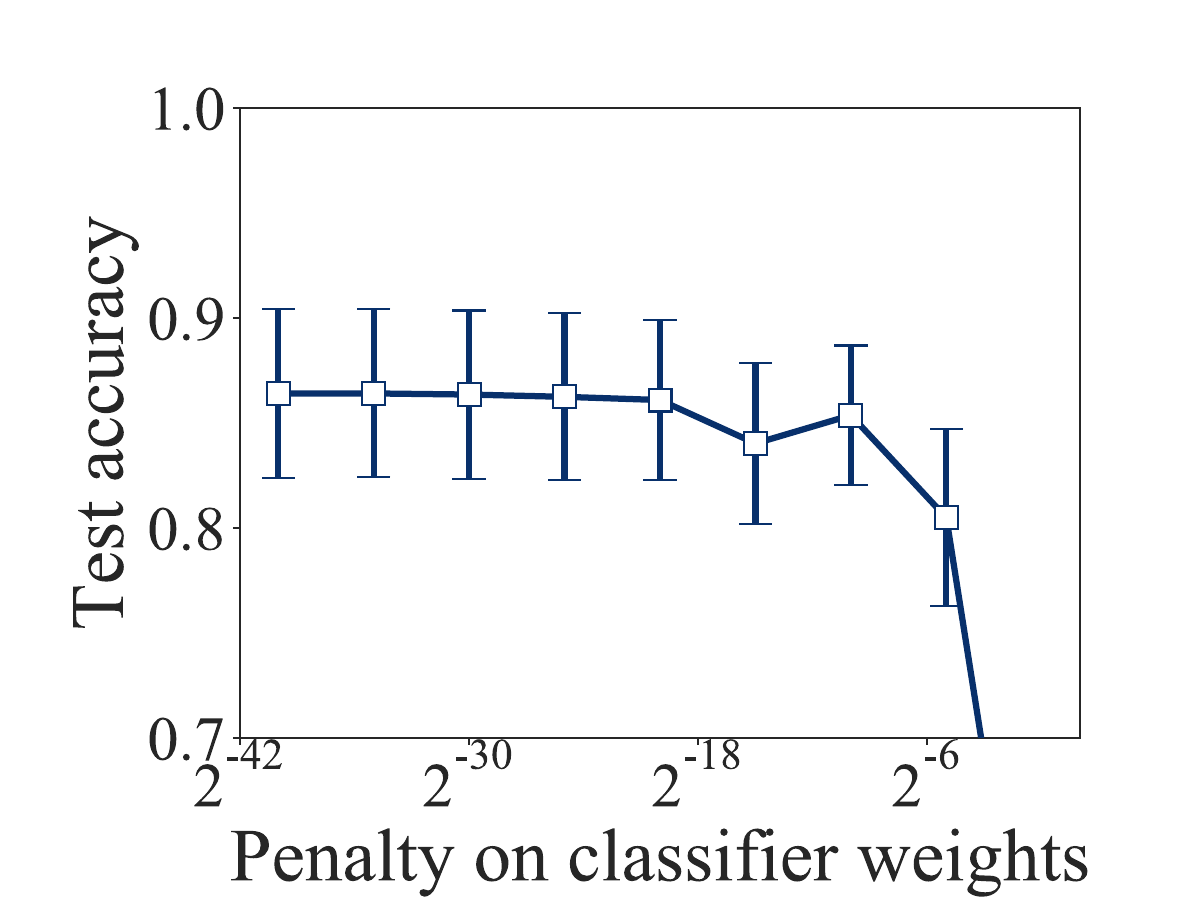}
\end{center}
\end{subfigure}
\caption{\label{fig:sensitivity} Sensitivity analysis of the hyperparameters tuned with hold-out validation when  using XSDC to train a LeNet-5 CKN on MNIST with 50 labeled observations.}
\end{figure}

\paragraph{Sensitivity to hyperparameters.}
The XSDC algorithm has three hyperparameters to tune in the semi-supervised case. In order to assess the importance of these parameters, we perform a sensitivity analysis, again for the LeNet-5 CKN on MNIST with 50 labeled observations. Figure~\ref{fig:sensitivity} displays the results when varying one parameter at a time, fixing the others to their values from hold-out validation. From the plots we can see that the parameter that requires the most careful tuning in this setting is the semi-supervised learning rate. The learning rate for the supervised initialization, along with the penalty on the classifier weights, just need to be sufficiently small.

\section{Conclusion}
In this work we presented a principled learning algorithm called XSDC that can be used on any amount of labeled and unlabeled data. In the special case of unsupervised learning the objective is a clustering objective in which the feature representation is also learned. In contrast, in the special case of supervised learning, the objective is a classification objective. We demonstrated the effectiveness of XSDC on four datasets, showing that when adding additional labeled data would help, substituting it with unlabeled data still often yields large performance improvements.

We designed our approach for situations in which the data can be processed in large batches and for partitioning problems such as unsupervised clustering or semi-supervised clustering. Going forward, it would be interesting to pursue a streaming version of this method that relaxes the large batch size requirement and that processes continuous streams of unlabeled and labeled data. It could also be interesting to think about whether something similar can be done for statistical problems with ordinal data or other data having a discrete structure. 
\section*{Acknowledgements}
The authors would like to thank the reviewers for their valuable comments that helped to improve the manuscript. 
The authors would like to gratefully acknowledge support from the National Science Foundation under grants NSF CCF-1740551 and NSF DMS-1810975, the program ``Learning in Machines and Brains'' of the Canadian Institute For Advanced Research, and faculty research awards. This work was first presented at the Women in Machine Learning Workshop in December 2019, for which the first author received travel funding from the National Science Foundation under grant NSF IIS-1833154. Part of this work was done while Corinne Jones was at the University of Washington. 

\bibliographystyle{agsm}
\bibliography{references,references_v2}

\begin{thebibliography}{97}
\providecommand{\natexlab}[1]{#1}
\providecommand{\url}[1]{{#1}}
\providecommand{\urlprefix}{URL }
\expandafter\ifx\csname urlstyle\endcsname\relax
  \providecommand{\doi}[1]{DOI~\discretionary{}{}{}#1}\else
  \providecommand{\doi}{DOI~\discretionary{}{}{}\begingroup
  \urlstyle{rm}\Url}\fi
\providecommand{\eprint}[2][]{\url{#2}}

\bibitem[{Alayrac et~al.(2016)Alayrac, Bojanowski, Agrawal, Sivic, Laptev, and
  Lacoste-Julien}]{alayrac2016}
Alayrac JB, Bojanowski P, Agrawal N, Sivic J, Laptev I, Lacoste-Julien S (2016)
  Unsupervised learning from narrated instruction videos. In: Conference on
  Computer Vision and Pattern Recognition, pp 4575--4583

\bibitem[{Ardila et~al.(2019)Ardila, Kiraly, Bharadwaj, Choi, Reicher, Peng,
  Tse, Etemadi, Ye, Corrado, Naidich, and Shetty}]{ardila2019}
Ardila D, Kiraly AP, Bharadwaj S, Choi B, Reicher JJ, Peng L, Tse D, Etemadi M,
  Ye W, Corrado G, Naidich DP, Shetty S (2019) End-to-end lung cancer screening
  with three-dimensional deep learning on low-dose chest computed tomography.
  Nat Med

\bibitem[{Asano et~al.(2020)Asano, Rupprecht, and Vedaldi}]{asano2020}
Asano YM, Rupprecht C, Vedaldi A (2020) Self-labelling via simultaneous
  clustering and representation learning. In: International Conference on
  Learning Representations

\bibitem[{Bach and Harchaoui(2007)}]{bach2007}
Bach FR, Harchaoui Z (2007) {DIFFRAC:} a discriminative and flexible framework
  for clustering. In: Advances in Neural Information Processing Systems, pp
  49--56

\bibitem[{Bach and Jordan(2006)}]{bach2006}
Bach FR, Jordan MI (2006) Learning spectral clustering, with application to
  speech separation. J Mach Learn Res 7:1963--2001

\bibitem[{Bachman et~al.(2014)Bachman, Alsharif, and Precup}]{bachman2014}
Bachman P, Alsharif O, Precup D (2014) Learning with pseudo-ensembles. In:
  Advances in Neural Information Processing Systems, pp 3365--3373

\bibitem[{Bachman et~al.(2019)Bachman, Hjelm, and Buchwalter}]{bachman2019}
Bachman P, Hjelm RD, Buchwalter W (2019) Learning representations by maximizing
  mutual information across views. In: Advances in Neural Information
  Processing Systems, pp 15509--15519

\bibitem[{Basu et~al.(2002)Basu, Banerjee, and Mooney}]{basu2002semi}
Basu S, Banerjee A, Mooney R (2002) Semi-supervised clustering by seeding. In:
  International Conference on Machine Learning, pp 27--34

\bibitem[{Belkin et~al.(2006)Belkin, Niyogi, and Sindhwani}]{belkin2006}
Belkin M, Niyogi P, Sindhwani V (2006) Manifold regularization: {A} geometric
  framework for learning from labeled and unlabeled examples. J Mach Learn Res
  7:2399--2434

\bibitem[{Belkin et~al.(2018)Belkin, Ma, and Mandal}]{belkin2018}
Belkin M, Ma S, Mandal S (2018) To understand deep learning we need to
  understand kernel learning. In: International Conference on Machine Learning,
  pp 540--548

\bibitem[{Berthelot et~al.(2019)Berthelot, Carlini, Goodfellow, Papernot,
  Oliver, and Raffel}]{berthelot2019}
Berthelot D, Carlini N, Goodfellow I, Papernot N, Oliver A, Raffel C (2019)
  Mix{M}atch: A holistic approach to semi-supervised learning. In: Advances in
  Neural Information Processing Systems, pp 5050--5060

\bibitem[{{Bertsekas}(2016)}]{bertsekas2016}
{Bertsekas} DP (2016) {Nonlinear programming}, 3rd edn. Athena Scientific

\bibitem[{Beyer et~al.(2019)Beyer, Zhai, Oliver, and Kolesnikov}]{beyer2019}
Beyer L, Zhai X, Oliver A, Kolesnikov A (2019) {S4L:} self-supervised
  semi-supervised learning. In: International Conference on Computer Vision, pp
  1476--1485

\bibitem[{Bilenko et~al.(2004)Bilenko, Basu, and Mooney}]{bilenko:etal:2004}
Bilenko M, Basu S, Mooney RJ (2004) Integrating constraints and metric learning
  in semi-supervised clustering. In: International Conference on Machine
  Learning

\bibitem[{Bo et~al.(2011)Bo, Lai, Ren, and Fox}]{bo2011}
Bo L, Lai K, Ren X, Fox D (2011) Object recognition with hierarchical kernel
  descriptors. In: Conference on Computer Vision and Pattern Recognition, pp
  1729--1736

\bibitem[{Bock et~al.(2004)Bock, Chilingarian, Gaug, Hakl, Hengstebeck, Jirina,
  Klaschka, Kotrc, Savicky, Towers, Vaicilius, and Wittek}]{bock2004}
Bock R, Chilingarian A, Gaug M, Hakl F, Hengstebeck T, Jirina M, Klaschka J,
  Kotrc E, Savicky P, Towers S, Vaicilius A, Wittek W (2004) {Methods for
  multidimensional event classification: A case study using images from a
  Cherenkov gamma-ray telescope}. Nucl Instrum Methods Phys Res A
  516(2):511--528

\bibitem[{Bojanowski and Joulin(2017)}]{bojanowski2017}
Bojanowski P, Joulin A (2017) Unsupervised learning by predicting noise. In:
  International Conference on Machine Learning, pp 517--526

\bibitem[{Bojanowski et~al.(2014)Bojanowski, Lajugie, Bach, Laptev, Ponce,
  Schmid, and Sivic}]{bojanowski2014}
Bojanowski P, Lajugie R, Bach F, Laptev I, Ponce J, Schmid C, Sivic J (2014)
  Weakly supervised action labeling in videos under ordering constraints. In:
  European Conference on Computer Vision, pp 628--643

\bibitem[{Bojanowski et~al.(2015)Bojanowski, Lajugie, Grave, Bach, Laptev,
  Ponce, and Schmid}]{bojanowski2015}
Bojanowski P, Lajugie R, Grave E, Bach F, Laptev I, Ponce J, Schmid C (2015)
  Weakly-supervised alignment of video with text. In: International Conference
  on Computer Vision, pp 4462--4470

\bibitem[{{Bouveyron} et~al.(2019){Bouveyron}, {Celeux}, {Murphy}, and
  {Raftery}}]{bouveyron2019}
{Bouveyron} C, {Celeux} G, {Murphy} TB, {Raftery} AE (2019) {Model-based
  clustering and classification for data science. With applications in R.}
  Cambridge: Cambridge University Press

\bibitem[{Byerly et~al.(2020)Byerly, Kalganova, and Dear}]{byerly2020}
Byerly A, Kalganova T, Dear I (2020) A branching and merging convolutional
  network with homogeneous filter capsules. CoRR abs/2001.09136

\bibitem[{Caron et~al.(2018)Caron, Bojanowski, Joulin, and Douze}]{caron2018}
Caron M, Bojanowski P, Joulin A, Douze M (2018) Deep clustering for
  unsupervised learning of visual features. In: European Conference on Computer
  Vision, pp 139--156

\bibitem[{Chang and Lin(2011)}]{chang2011}
Chang CC, Lin CJ (2011) {LIBSVM}: A library for support vector machines. {ACM}
  Trans Intell Syst Technol 2:27:1--27:27

\bibitem[{Chapelle et~al.(2010)Chapelle, Schlkopf, and Zien}]{chapelle2010}
Chapelle O, Schlkopf B, Zien A (2010) Semi-Supervised Learning, 1st edn. The
  MIT Press

\bibitem[{Dahlhaus et~al.(1994)Dahlhaus, Johnson, Papadimitriou, Seymour, and
  Yannakakis}]{dahlhaus1994complexity}
Dahlhaus E, Johnson DS, Papadimitriou CH, Seymour PD, Yannakakis M (1994) The
  complexity of multiterminal cuts. SIAM Journal on Computing 23(4):864--894

\bibitem[{Daniely et~al.(2016)Daniely, Frostig, and Singer}]{daniely2016}
Daniely A, Frostig R, Singer Y (2016) Toward deeper understanding of neural
  networks: the power of initialization and a dual view on expressivity. In:
  Advances in Neural Information Processing Systems, pp 2253--2261

\bibitem[{Daniely et~al.(2017)Daniely, Frostig, Gupta, and
  Singer}]{daniely2017}
Daniely A, Frostig R, Gupta V, Singer Y (2017) Random features for
  compositional kernels. CoRR abs/1703.07872

\bibitem[{Doersch et~al.(2015)Doersch, Gupta, and Efros}]{doersch2015}
Doersch C, Gupta A, Efros AA (2015) Unsupervised visual representation learning
  by context prediction. In: International Conference on Computer Vision, pp
  1422--1430

\bibitem[{Dosovitskiy et~al.(2016)Dosovitskiy, Fischer, Springenberg,
  Riedmiller, and Brox}]{dosovitskiy2016}
Dosovitskiy A, Fischer P, Springenberg JT, Riedmiller MA, Brox T (2016)
  Discriminative unsupervised feature learning with exemplar convolutional
  neural networks. {IEEE} Trans Pattern Anal Mach Intell 38(9):1734--1747

\bibitem[{van Engelen and Hoos(2020)}]{vanengelen2020}
van Engelen JE, Hoos HH (2020) A survey on semi-supervised learning. Mach Learn
  109(2):373--440

\bibitem[{Flammarion et~al.(2017)Flammarion, Palaniappan, and
  Bach}]{flammarion2017}
Flammarion N, Palaniappan B, Bach F (2017) Robust discriminative clustering
  with sparse regularizers. J Mach Learn Res 18(80):1--50

\bibitem[{Fukumizu et~al.(2009)Fukumizu, Gretton, Lanckriet, Sch\"{o}lkopf, and
  Sriperumbudur}]{fukumizu2009kernel}
Fukumizu K, Gretton A, Lanckriet G, Sch\"{o}lkopf B, Sriperumbudur BK (2009)
  Kernel choice and classifiability for rkhs embeddings of probability
  distributions. In: Advances in Neural Information Processing Systems

\bibitem[{{Ghasedi Dizaji} et~al.(2017){Ghasedi Dizaji}, Herandi, Deng, Cai,
  and Huang}]{ghasedi2017}
{Ghasedi Dizaji} K, Herandi A, Deng C, Cai W, Huang H (2017) Deep clustering
  via joint convolutional autoencoder embedding and relative entropy
  minimization. In: International Conference on Computer Vision, pp 5747--5756

\bibitem[{Goodfellow et~al.(2016)Goodfellow, Bengio, and
  Courville}]{goodfellow2016}
Goodfellow IJ, Bengio Y, Courville AC (2016) Deep Learning. Adaptive
  computation and machine learning, {MIT} Press

\bibitem[{Grandvalet and Bengio(2004)}]{grandvalet2004}
Grandvalet Y, Bengio Y (2004) Semi-supervised learning by entropy minimization.
  In: Advances in Neural Information Processing Systems, pp 529--536

\bibitem[{Guyon et~al.(2004)Guyon, Gunn, Ben{-}Hur, and Dror}]{guyon2004}
Guyon I, Gunn SR, Ben{-}Hur A, Dror G (2004) Result analysis of the {NIPS} 2003
  feature selection challenge. In: Advances in Neural Information Processing
  Systems, pp 545--552

\bibitem[{Hagen and Kahng(1992)}]{hagen1992}
Hagen LW, Kahng AB (1992) New spectral methods for ratio cut partitioning and
  clustering. {IEEE} Trans Comput Aided Des Integr Circuits Syst
  11(9):1074--1085

\bibitem[{H{\"{a}}usser et~al.(2017)H{\"{a}}usser, Mordvintsev, and
  Cremers}]{hausser2017}
H{\"{a}}usser P, Mordvintsev A, Cremers D (2017) Learning by association - {A}
  versatile semi-supervised training method for neural networks. In: Conference
  on Computer Vision and Pattern Recognition, pp 626--635

\bibitem[{Hennig et~al.(2015)Hennig, Meila, Murtagh, and
  Rocci}]{hennig2015handbook}
Hennig C, Meila M, Murtagh F, Rocci R (2015) Handbook of Cluster Analysis.
  Chapman \& Hall/CRC Handbooks of Modern Statistical Methods, CRC Press

\bibitem[{Hyv{\"{a}}rinen and Morioka(2016)}]{hyvarinen2016}
Hyv{\"{a}}rinen A, Morioka H (2016) Unsupervised feature extraction by
  time-contrastive learning and nonlinear {ICA}. In: Advances in Neural
  Information Processing Systems, pp 3765--3773

\bibitem[{Iscen et~al.(2019)Iscen, Tolias, Avrithis, and Chum}]{iscen2019}
Iscen A, Tolias G, Avrithis Y, Chum O (2019) Label propagation for deep
  semi-supervised learning. In: Conference on Computer Vision and Pattern
  Recognition, pp 5070--5079

\bibitem[{Jalali et~al.(2016)Jalali, Han, Dumitriu, and Fazel}]{jalali2016}
Jalali A, Han Q, Dumitriu I, Fazel M (2016) Exploiting tradeoffs for exact
  recovery in heterogeneous stochastic block models. In: Advances in Neural
  Information Processing Systems, pp 4871--4879

\bibitem[{{Johnson} et~al.(2019){Johnson}, {Douze}, and
  {J{\'{e}}gou}}]{johnson2019}
{Johnson} J, {Douze} M, {J{\'{e}}gou} H (2019) Billion-scale similarity search
  with {GPU}s. {IEEE} Trans Big Data (early access)

\bibitem[{Jones(2020)}]{jones2020}
Jones C (2020) Representation learning for partitioning problems. PhD thesis,
  University of Washington

\bibitem[{Joulin and Bach(2012)}]{joulin2012}
Joulin A, Bach FR (2012) A convex relaxation for weakly supervised classifiers.
  In: International Conference on Machine Learning, pp 1315--1322

\bibitem[{Joulin et~al.(2010)Joulin, Bach, and Ponce}]{joulin2010}
Joulin A, Bach FR, Ponce J (2010) Discriminative clustering for image
  co-segmentation. In: Conference on Computer Vision and Pattern Recognition,
  pp 1943--1950

\bibitem[{Kamnitsas et~al.(2018)Kamnitsas, Castro, Folgoc, Walker, Tanno,
  Rueckert, Glocker, Criminisi, and Nori}]{kamnitsas2018}
Kamnitsas K, Castro DC, Folgoc LL, Walker I, Tanno R, Rueckert D, Glocker B,
  Criminisi A, Nori AV (2018) Semi-supervised learning via compact latent space
  clustering. In: International Conference on Machine Learning, pp 2464--2473

\bibitem[{{Karp}(1975)}]{karp1975}
{Karp} RM (1975) {Reducibility among combinatorial problems.} {Kibern Sb, Nov
  Ser} 12:16--38

\bibitem[{Krizhevsky and Hinton(2009)}]{krizhevsky2009}
Krizhevsky A, Hinton G (2009) Learning multiple layers of features from tiny
  images. Tech. rep., University of Toronto

\bibitem[{Law et~al.(2017)Law, Urtasun, and Zemel}]{law2017}
Law MT, Urtasun R, Zemel RS (2017) Deep spectral clustering learning. In:
  International Conference on Machine Learning, pp 1985--1994

\bibitem[{LeCun(1987)}]{lecun1987}
LeCun Y (1987) Modeles connexionnistes de l'apprentissage. PhD thesis,
  Universit\'e P. et M. Curie (Paris 6)

\bibitem[{LeCun et~al.(2001)LeCun, Bottou, Bengio, and Haffner}]{lecun2001}
LeCun Y, Bottou L, Bengio Y, Haffner P (2001) Gradient-based learning applied
  to document recognition. In: Intelligent Signal Processing, IEEE Press, pp
  306--351

\bibitem[{Lee(2013)}]{lee2013}
Lee DH (2013) Pseudo-label: The simple and efficient semi-supervised learning
  method for deep neural networks. In: International Conference on Machine
  Learning Workshop on Challenges in Representation Learning

\bibitem[{Lee et~al.(2018)Lee, Bahri, Novak, Schoenholz, Pennington, and
  Sohl{-}Dickstein}]{lee2018}
Lee J, Bahri Y, Novak R, Schoenholz SS, Pennington J, Sohl{-}Dickstein J (2018)
  Deep neural networks as {G}aussian processes. In: International Conference on
  Learning Representations

\bibitem[{Li et~al.(2018)Li, Wang, Ji, Xiang, and Fox}]{li2018}
Li Y, Wang G, Ji X, Xiang Y, Fox D (2018) Deep{IM}: Deep iterative matching for
  6{D} pose estimation. In: European Conference on Computer Vision, pp 695--711

\bibitem[{L\"{o}we et~al.(2019)L\"{o}we, O'Connor, and Veeling}]{lowe2019}
L\"{o}we S, O'Connor P, Veeling B (2019) Putting an end to end-to-end:
  Gradient-isolated learning of representations. In: Advances in Neural
  Information Processing Systems, pp 3033--3045

\bibitem[{{L\"utkepohl}(1996)}]{lutkepohl1996}
{L\"utkepohl} H (1996) {Handbook of matrices.} Chichester: John Wiley \& Sons

\bibitem[{von Luxburg(2007)}]{vonluxburg2007}
von Luxburg U (2007) A tutorial on spectral clustering. Stat Comput
  17(4):395--416

\bibitem[{{MacQueen}(1967)}]{macqueen1967}
{MacQueen} J (1967) {Some methods for classification and analysis of
  multivariate observations.} In: {Berkeley Symposium on Mathematical
  Statistics and Probability}

\bibitem[{Mairal(2016)}]{mairal2016}
Mairal J (2016) End-to-end kernel learning with supervised convolutional kernel
  networks. In: Advances in Neural Information Processing Systems, pp
  1399--1407

\bibitem[{Mairal et~al.(2014)Mairal, Koniusz, Harchaoui, and
  Schmid}]{mairal2014}
Mairal J, Koniusz P, Harchaoui Z, Schmid C (2014) Convolutional kernel
  networks. In: Advances in Neural Information Processing Systems, pp
  2627--2635

\bibitem[{{Matthews} et~al.(2018){Matthews}, Hron, Rowland, Turner, and
  Ghahramani}]{matthews2018}
{Matthews} A, Hron J, Rowland M, Turner RE, Ghahramani Z (2018) Gaussian
  process behaviour in wide deep neural networks. In: International Conference
  on Learning Representations

\bibitem[{McQueen et~al.(2016)McQueen, Meil{\u{a}}, VanderPlas, and
  Zhang}]{megaman}
McQueen J, Meil{\u{a}} M, VanderPlas J, Zhang Z (2016) Megaman: Scalable
  manifold learning in {P}ython. Journal of Machine Learning Research
  17(148):1--5

\bibitem[{{Meila}(2016)}]{meila2016}
{Meila} M (2016) {Spectral clustering}. In: {Handbook of cluster analysis},
  Boca Raton, FL: CRC Press, pp 125--141

\bibitem[{Meila et~al.(2005)Meila, Shortreed, and Xu}]{meila2005}
Meila M, Shortreed SM, Xu L (2005) Regularized spectral learning. In: Workshop
  on Artificial Intelligence and Statistics

\bibitem[{Mohri et~al.(2012)Mohri, Rostamizadeh, and Talwalkar}]{mohri2012}
Mohri M, Rostamizadeh A, Talwalkar A (2012) Foundations of Machine Learning.
  Adaptive computation and machine learning, {MIT} Press

\bibitem[{{Nesterov}(2018)}]{nesterov2018}
{Nesterov} Y (2018) {Lectures on convex optimization}, 2nd edn. Springer

\bibitem[{Noroozi and Favaro(2016)}]{noroozi2016}
Noroozi M, Favaro P (2016) Unsupervised learning of visual representations by
  solving jigsaw puzzles. In: European Conference on Computer Vision, pp 69--84

\bibitem[{Oglic and G{\"{a}}rtner(2017)}]{oglic2017}
Oglic D, G{\"{a}}rtner T (2017) Nystr{\"{o}}m method with kernel k-means++
  samples as landmarks. In: International Conference on Machine Learning, pp
  2652--2660

\bibitem[{Oliver et~al.(2018)Oliver, Odena, Raffel, Cubuk, and
  Goodfellow}]{oliver2018}
Oliver A, Odena A, Raffel CA, Cubuk ED, Goodfellow IJ (2018) Realistic
  evaluation of deep semi-supervised learning algorithms. In: Advances in
  Neural Information Processing Systems, pp 3239--3250

\bibitem[{Paszke et~al.(2019)Paszke, Gross, Massa, Lerer, Bradbury, Chanan,
  Killeen, Lin, Gimelshein, Antiga, Desmaison, Kopf, Yang, DeVito, Raison,
  Tejani, Chilamkurthy, Steiner, Fang, Bai, and Chintala}]{paszke2019}
Paszke A, Gross S, Massa F, Lerer A, Bradbury J, Chanan G, Killeen T, Lin Z,
  Gimelshein N, Antiga L, Desmaison A, Kopf A, Yang E, DeVito Z, Raison M,
  Tejani A, Chilamkurthy S, Steiner B, Fang L, Bai J, Chintala S (2019)
  Py{T}orch: An imperative style, high-performance deep learning library. In:
  Advances in Neural Information Processing Systems, pp 8024--8035

\bibitem[{{Perez-Cruz} and {Bousquet}(2004)}]{perez-cruz2004}
{Perez-Cruz} F, {Bousquet} O (2004) Kernel methods and their potential use in
  signal processing. {IEEE} Signal Process Mag 21(3):57--65

\bibitem[{Peyr{\'{e}} and Cuturi(2019)}]{peyre2019}
Peyr{\'{e}} G, Cuturi M (2019) Computational optimal transport. Foundations and
  Trends in Machine Learning 11(5-6):355--607

\bibitem[{Rahimi and Recht(2007)}]{rahimi2007}
Rahimi A, Recht B (2007) Random features for large-scale kernel machines. In:
  Advances in Neural Information Processing Systems, pp 1177--1184

\bibitem[{Sch{\"o}lkopf et~al.(1998)Sch{\"o}lkopf, Smola, and
  M{\"u}ller}]{scholkopf1998}
Sch{\"o}lkopf B, Smola A, M{\"u}ller KR (1998) Nonlinear component analysis as
  a kernel eigenvalue problem. Neural Comput 10(5):1299--1319

\bibitem[{Schrijver(2003)}]{schrijver2003}
Schrijver A (2003) Combinatorial optimization. {P}olyhedra and efficiency.
  {V}ol. {A}, Algorithms and Combinatorics, vol~24. Springer-Verlag, Berlin

\bibitem[{Sermanet et~al.(2018)Sermanet, Lynch, Chebotar, Hsu, Jang, Schaal,
  and Levine}]{sermanet2018}
Sermanet P, Lynch C, Chebotar Y, Hsu J, Jang E, Schaal S, Levine S (2018)
  Time-contrastive networks: Self-supervised learning from video. In:
  International Conference on Robotics and Automation, pp 1134--1141

\bibitem[{Shi and Malik(2000)}]{shi2000}
Shi J, Malik J (2000) Normalized cuts and image segmentation. {IEEE} Trans
  Pattern Anal Mach Intell 22(8):888--905

\bibitem[{{Sinkhorn} and {Knopp}(1967)}]{sinkhorn1967}
{Sinkhorn} R, {Knopp} P (1967) {Concerning nonnegative matrices and doubly
  stochastic matrices.} {Pac J Math} 21:343--348

\bibitem[{Swamy(2004)}]{swamy2004}
Swamy C (2004) Correlation clustering: maximizing agreements via semidefinite
  programming. In: {ACM-SIAM} Symposium on Discrete Algorithms, pp 526--527

\bibitem[{Thickstun et~al.(2018)Thickstun, Harchaoui, Foster, and
  Kakade}]{thickstun2018}
Thickstun J, Harchaoui Z, Foster DP, Kakade SM (2018) Invariances and data
  augmentation for supervised music transcription. In: International Conference
  on Acoustics, Speech and Signal Processing, pp 2241--2245

\bibitem[{{Van Der Maaten} and {Hinton}(2008)}]{vandermaaten2008}
{Van Der Maaten} L, {Hinton} G (2008) {Visualizing data using t-SNE.} {J Mach
  Learn Res} 9:2579--2605

\bibitem[{Virtanen et~al.(2020)Virtanen, Gommers, Oliphant, Haberland, Reddy,
  Cournapeau, Burovski, Peterson, Weckesser, Bright, van~der Walt, Brett,
  Wilson, Millman, Mayorov, Nelson, Jones, Kern, Larson, ..., and
  V{\'a}zquez-Baeza}]{virtanen2020}
Virtanen P, Gommers R, Oliphant TE, Haberland M, Reddy T, Cournapeau D,
  Burovski E, Peterson P, Weckesser W, Bright J, van~der Walt SJ, Brett M,
  Wilson J, Millman KJ, Mayorov N, Nelson ARJ, Jones E, Kern R, Larson E, ,
  V{\'a}zquez-Baeza Y (2020) Scipy 1.0: fundamental algorithms for scientific
  computing in {P}ython. Nat Methods 17(3):261--272

\bibitem[{{Vrbik} and {McNicholas}(2015)}]{vrbik2015}
{Vrbik} I, {McNicholas} PD (2015) {Fractionally-supervised classification.} {J
  Classif} 32(3):359--381

\bibitem[{Wang and Gupta(2015)}]{wang2015}
Wang X, Gupta A (2015) Unsupervised learning of visual representations using
  videos. In: International Conference on Computer Vision, pp 2794--2802

\bibitem[{White and Schuurmans(2012)}]{white2012}
White M, Schuurmans D (2012) Generalized optimal reverse prediction. In:
  International Conference on Artificial Intelligence and Statistics, pp
  1305--1313

\bibitem[{Williams and Seeger(2000)}]{williams2000}
Williams CKI, Seeger MW (2000) Using the {N}ystr{\"{o}}m method to speed up
  kernel machines. In: Advances in Neural Information Processing Systems, pp
  682--688

\bibitem[{Wu and Leahy(1993)}]{wu1993}
Wu Z, Leahy RM (1993) An optimal graph theoretic approach to data clustering:
  Theory and its application to image segmentation. {IEEE} Trans Pattern Anal
  Mach Intell 15(11):1101--1113

\bibitem[{Wu et~al.(2018)Wu, Xiong, Yu, and Lin}]{wu2018}
Wu Z, Xiong Y, Yu SX, Lin D (2018) Unsupervised feature learning via
  non-parametric instance discrimination. In: Conference on Computer Vision and
  Pattern Recognition, pp 3733--3742

\bibitem[{Xie et~al.(2016)Xie, Girshick, and Farhadi}]{xie2016}
Xie J, Girshick RB, Farhadi A (2016) Unsupervised deep embedding for clustering
  analysis. In: International Conference on Machine Learning, pp 478--487

\bibitem[{Xing and Jordan(2003)}]{xing2003}
Xing EP, Jordan MI (2003) On semidefinite relaxation for normalized k-cut and
  connections to spectral clustering. Tech. Rep. UCB/CSD-03-1265, EECS
  Department, University of California, Berkeley

\bibitem[{Xu et~al.(2009)Xu, White, and Schuurmans}]{xu2009}
Xu L, White M, Schuurmans D (2009) Optimal reverse prediction: a unified
  perspective on supervised, unsupervised and semi-supervised learning. In:
  International Conference on Machine Learning, pp 1137--1144

\bibitem[{Yang et~al.(2016)Yang, Parikh, and Batra}]{yang2016}
Yang J, Parikh D, Batra D (2016) Joint unsupervised learning of deep
  representations and image clusters. In: Conference on Computer Vision and
  Pattern Recognition, pp 5147--5156

\bibitem[{Yoder and Priebe(2017)}]{yoder2017semi}
Yoder J, Priebe CE (2017) Semi-supervised {$k$-means$++$}. Journal of
  Statistical Computation and Simulation 87(13):2597--2608

\bibitem[{Zass and Shashua(2006)}]{zass2006}
Zass R, Shashua A (2006) Doubly stochastic normalization for spectral
  clustering. In: Advances in Neural Information Processing Systems, pp
  1569--1576

\bibitem[{Zha et~al.(2001)Zha, He, Ding, Gu, and Simon}]{zha2001}
Zha H, He X, Ding CHQ, Gu M, Simon HD (2001) Spectral relaxation for k-means
  clustering. In: Advances in Neural Information Processing Systems, pp
  1057--1064

\bibitem[{Zhang et~al.(2016)Zhang, Isola, and Efros}]{zhang2016}
Zhang R, Isola P, Efros AA (2016) Colorful image colorization. In: European
  Conference on Computer Vision, pp 649--666

\end{thebibliography}

\clearpage
\appendix
\section*{\LARGE Appendix}
\section{Smoothness of the Objective Function}
\label{app:conditioning}
Recall the forward and reverse prediction objectives, which, given a set of labels $Y \in \{0,1\}^{n\times k}$ with $Y \ones_k = \ones_n$ and a feature representation $\Phi$, read respectively
\begin{align*}
	F_f(\Phi) & = \min_{W, b} \frac{1}{n}\|Y-\Phi W -\ones_n b^T\|_F^2+\regtwo\|W\|^2_F =   \regtwo\tr[YY^T\Pi_n(\Pi_n\Phi\Phi^T\Pi_n + n\regtwo \id)^{-1}\Pi_n] \\
	F_r(\Phi) & = \min_W \frac{1}{n}\|\Phi(V)-YW\|^2_F = \frac{1}{n}\tr[(\id-P_Y)\Phi\Phi^T]\;,
\end{align*}
where $\Pi_n = \id_n - \ones_n\ones_n^T/n$ and $P_Y=Y(Y^TY)^{-1}Y^T$ are orthonormal projectors. In this appendix we provide the proofs of Propositions~\ref{prop:smoothness_obj} and~\ref{prop:smoothness}, which estimate the smoothness constants of the above objectives.

\begin{proof}[Proof of Proposition~\ref{prop:smoothness_obj}]
	The gradient of the forward prediction objective is for $\Phi \in \mathcal{Z}$,
	\begin{align}
		\nabla F_f(\Phi)  = -2\regtwo \Pi_nG(\Phi) \Pi_nYY^T\Pi_n G(\Phi) \Pi_n\Phi\;, \label{eq:grad_forward_pred}
	\end{align}
	where $G(\Phi) = \left(\Pi_n\Phi \Phi^T\Pi_n{+} n\regtwo\id_{n}\right)^{-1}$.
	Since $\|\Pi_n\|_2 \leq 1$, $\|G(\Phi)\|_2\leq 1/(n\regtwo)$,  $\|YY^T\|_2\leq n_\text{max}$, where $n_{\max}$ is the maximal cluster size, and $\|\Phi\|_2 \leq B$ by assumption, 
	we obtain for any $\Phi \in \mathcal{Z}$, 
	\begin{align*}
		\|\nabla F_f(\Phi) \|_2 & \leq   \frac{2Bn_{\max}}{n^2\regtwo}=   \frac{2B\rho_{\max} }{n\regtwo} \eqqcolon L_f,
	\end{align*}
	where $\rho_{\max} = n_{\max}/n$.
	The gradient of the reverse prediction objective is
	$\nabla F_r(\Phi) = \frac{2}{n}(\id-P_Y)\Phi$,
	and can  be bounded as, for any $\Phi \in \mathcal{Z}$, 
	\begin{align}
		\|\nabla F_r(\Phi)\|_2 \leq 2B/n \eqqcolon L_r\;.
		\nonumber
	\end{align}
	Hence, taking $\regtwo \geq \rho_{\max}$, we have $L_f\leq L_r$. 
\end{proof}

\begin{proof}[Proof of Proposition~\ref{prop:smoothness}]
	Let $\Phi_1, \Phi_2 \in \mathcal{Z}$, denote $\bar Y = \Pi_n Y$, $\bar \Phi_1 = \Pi_n \Phi_1$, $\bar G_1 = \Pi_n G(\Phi_1)$, with $\bar \Phi_2$, $\bar G_2$ defined analogously and $G(\Phi)$ defined in the proof of Proposition~\ref{prop:smoothness_obj}. We decompose the difference of the gradients of the forward prediction defined in~\eqref{eq:grad_forward_pred} as 
	\begin{align*}
		-\frac{1}{2\lambda}(\nabla F_f(\Phi_1) {-} \nabla F_f(\Phi_2)) &= \bar G_1 \bar Y \bar Y^\top \bar G_1 \bar \Phi_1 - \bar G_2 \bar Y \bar Y^\top \bar G_2  \bar \Phi_2\\
		& = (\bar G_1 \bar Y \bar Y^\top \bar G_1 - \bar G_2 \bar Y \bar Y^\top \bar G_2) \bar \Phi_1 + \bar G_2 \bar Y \bar Y^\top \bar G_2 (\bar \Phi_1 - \bar \Phi_2) \\
		\bar G_1 \bar Y \bar Y^\top \bar G_1 - \bar G_2 \bar Y \bar Y^\top \bar G_2 & = \bar G_1\bar Y\bar Y^\top( \bar G_1 -\bar G_2) + (\bar G_1- \bar G_2)\bar Y\bar Y^\top \bar G_2\\
		G_1 -  G_2 & =  G_1 (G_2^{-1} -  G_1^{-1})G_2 \\
		G_2^{-1} -  G_1^{-1} & = \bar \Phi_2 \bar\Phi_2^\top   - \bar \Phi_1 \bar \Phi_1^\top = \bar \Phi_2 (\bar \Phi_2 - \bar \Phi_1)^\top + (\bar \Phi_2 - \bar \Phi_1) \bar \Phi_1^\top.
	\end{align*}
	The difference can readily be bounded using that (i) $\| \bar Y\bar Y^\top \|_2 \leq \|YY^\top \|_2 \leq n_{\max}$ with $n_{\max}$ the maximal cluster size,  (ii) $\|\bar \Phi_1 - \bar\Phi_2\|_2 \leq \|\Phi_1 - \Phi_2\|_2$, $\|\bar G_1 - \bar G_2\|_2 \leq \|G_1 - G_2\|_2$, and (iii) for $i \in \{1, 2\}$, $\|\bar G_i\|_2 \leq {1}/{(n\lambda)}$, $\|\bar \Phi_i\|_2\leq  B$ by assumption. Hence, we get 
	\begin{align*}
		\frac{1}{2\regtwo}\frac{\left\Vert \nabla F_f(\Avgfeat_1){-}\nabla F_f(\Avgfeat_2)\right\Vert_2}{\|\Avgfeat_1{-}\Avgfeat_2\|_2}
		\leq\; \left(\frac{4B^2n_{\max}}{n^3\regtwo^3} {+} \frac{n_{\max}}{n^2\regtwo^2}\right)\;,
	\end{align*}
	and so an upper bound on the Lipschitz constant of the gradients of the forward prediction objective  on $\mathcal{Z}$ is $\ell_f \coloneqq
	2\rho_{\max}/(n\regtwo) + 8B^2\rho_{\max}/(n\regtwo)^2
	$, where $\rho_{\max} = n_{\max}/n$.	
	For the reverse prediction objective, we have
	\begin{align*}
		\nabla F_r(\Avgfeat_1)-F_r(\Avgfeat_2)  = \frac{2}{n} (\id - P_Y) (\Phi_1 - \Phi_2)\;, 
	\end{align*}
	where $P_Y=Y(Y^TY)^{-1}Y^T$ is an orthonormal projector. Hence the Lipschitz constant of the gradients of the reverse prediction objective is at most $\ell_r := 2/n$.  	
	For $\regtwo \geq
	(\rho_{\max} + \sqrt{\rho_{\max}^2{+}16B^2\rho_{\max}})/2
	$, we therefore have $\ell_f\leq\ell_r$.
\end{proof}
\section{NP-Completeness of the Label Assignment Problem}
\label{app:balancing}
Now we address the problem of optimizing the labels for the unlabeled data. The following proposition shows that this discrete problem is in general NP-complete for $k>2$. Similar results were shown by~\cite{dahlhaus1994complexity}.

\begin{proposition}
\label{prop:balancing_npcomplete}
Let $A\in\mbr^{n\times n}$. The label assignment problem 
\begin{align*}
\min_Y\; &\tr(YY^TA)\\
\mbox{s.t.}\; &\sum_{j=1}^k Y_{ij}=1, \quad i=1,\dots, n, \ Y_{ij}\in\{0,1\} \quad \forall \; i=1,\dots, n,\; j=1,\dots, k
\end{align*}
is NP-complete for $k>2$.
\end{proposition}

\begin{proof}
The proof will follow by showing that the $k$-coloring problem is a special case of the matrix balancing problem. Let $G$ be an undirected, unweighted graph with no self-loops. Define $A\in\{0,1\}^{n\times n}$ to be the adjacency matrix of $G$. Then $G$  is $k$-colorable if and only if the following problem has minimum value zero:
\begin{align*}
\min_Y\; &\sum_{j=1}^k \sum_{i,i'\in A} Y_{i,j}Y_{i',j} \\
\mbox{s.t.}\; &\sum_{j=1}^k Y_{ij}=1, \quad i=1,\dots, n, \ Y_{ij}\in\{0,1\} \quad \forall \; i=1,\dots, n, \; j=1,\dots, k\;.
\end{align*}
Noting that 
$
\sum_{j=1}^k \sum_{i,i'\in A} Y_{i,j}Y_{i',j} = \tr(YY^TA),
$
we may rewrite the above problem as 
\begin{align*}
\min_Y\; &\tr(YY^TA) \\*
\mbox{s.t.}\; &\sum_{j=1}^k Y_{ij}=1, \quad \forall \; i=1,\dots, n, \ Y_{ij}\in\{0,1\} \quad \forall \; i=1,\dots, n,\;  j=1,\dots, k\;.
\end{align*}
This is a special case of the matrix balancing problem, in which $A$ is the adjacency matrix of a graph. Therefore, as the $k$-coloring problem is NP-complete for $k>2$ \citep{karp1975}, the label assignment problem with discrete assignments is also NP-complete for $k>2$.
\end{proof}

\section{An Alternative Relaxation}
\label{app:alt_label}
\citet{bach2007} propose alternative relaxations of the labeling subproblem. Define $\lambda_1 \leq \lambda_2 \leq \cdots \leq \lambda_n$ to be the eigenvalues of the equivalence matrix $M$ and let $\lambda_0\geq 0$. In Section 2.6 of their paper Bach and Harchaoui suggest solving the problem
\begin{align}
\min_{M \in \reals^{n \times n}} \quad  &
\Tr(M^T A) \label{eq:bh_relaxation} \\
\mbox{subject to} \quad &   M=M^T, \ \tr(M) = n, \ M \succeq 0, \ \sum_{i=1}^n \min\left\{\frac{\lambda_i}{\lambda_0}, 1\right\} \geq k \nonumber
\end{align}
in the unsupervised setting, for $A$ a symmetric matrix. 

\subsection{Derivation of the solution}
\begin{proposition}
	A solution for problem~\eqref{eq:bh_relaxation} is given by 
	\[
	M^\star=\sum_{i=1}^n \lambda_i^\star u_iu_i^T\;,
	\]
	where $u_1,\dots, u_n$ are eigenvectors corresponding to the eigenvalues $a_1\leq a_2,\dots\leq a_n$ of $A$ and 
	\begin{itemize}
		\item if $n>k$,  $\lambda_1^\star=n-(k-1)\lambda_0$, $\lambda_2^\star = \ldots =  \lambda_k^\star=\lambda_0$, $\lambda_{k+1}^\star= \ldots = \lambda_n^\star=0$,
		\item if $n=k$, $\lambda_1^\star,\dots, \lambda_k^\star=1$.
	\end{itemize}
\end{proposition}
\begin{proof}
Note that the symmetric and positive semi-definite constraints imply that we can write $M=U\Lambda U^T$ where $U$ contains an orthonormal set of eigenvectors of $M$ and $\Lambda\geq0$ is a diagonal matrix containing the corresponding eigenvalues. After rewriting $\Tr(M^T A)=\sum_{i=1}^n \lambda_iu_i^TAu_i$, with $u_i = U_{\cdot, i}$, we obtain the problem
\begin{align*}
\min_{U \in \reals^{n \times n}, \lambda_1,\dots, \lambda_n\in\mbr} \quad  &
\sum_{i=1}^n \lambda_iu_i^TAu_i \\
\mbox{subject to} \quad & \sum_{i=1}^n \lambda_i = n, \ \sum_{i=1}^n \min\left\{\frac{\lambda_i}{\lambda_0}, 1\right\} \geq k, \ u_i^Tu_i = 1 \ \forall i, \  u_i^Tu_j = 0 \ \forall i\neq j, \ \lambda_i \geq 0 \quad \forall i \;.
\end{align*}
Introducing Lagrange multipliers and defining the Lagrangian 
\begin{align*}
\mcl(U, \Lambda, \alpha, \beta, \gamma, \delta, \epsilon) 
& = \sum_{i=1}^n \lambda_iu_i^TAu_i + \alpha\left(\sum_{i=1}^n \lambda_i - n\right) - \beta\left(\sum_{i=1}^n \min\left\{\frac{\lambda_i}{\lambda_0}, 1\right\} - k\right) \\
& \quad + \sum_{i=1}^n \gamma_i(u_i^Tu_i - 1) + \sum_{i\neq j} \delta_{ij} u_i^Tu_j - \sum_{i=1}^n \epsilon_i\lambda_i \;, 
\end{align*}
we can rewrite the problem as
\begin{align*}
\max_{\alpha\in\mbr, \beta\in\mbr, \gamma\in\mbr^n, \delta\in\mbr^{n^2}} \min_{U \in \reals^{n \times n}, \lambda_1,\dots, \lambda_n\in\mbr} \quad  &
\mcl(U, \Lambda, \alpha, \beta, \gamma, \delta, \epsilon) \\
\mbox{subject to} \quad & \beta \geq 0, \;\; \epsilon_i \geq 0  \;\;\; \forall i\;,
\end{align*}
where $\alpha\in\mbr, \beta\in\mbr, \gamma\in\mbr^n, \delta\in\mbr^{n^2}, \epsilon\in\mbr^n$ and we define $\delta_{ii}=0$ for all $i$. The optimal parameter values must satisfy the first order conditions
\begin{align}
&2\lambda_i^\star Au_i^\star + 2\gamma_i^\star u_i^\star + \sum_{i\neq j} \delta_{ij}^\star u_j^\star = 0 \quad \forall i \label{eq:alt_label_foc2}\\
& {u_i^\star}^TAu_i^\star  + \alpha^\star  - \beta^\star \left[\frac{1}{2\lambda_0}(1-\text{sign}(\lambda_i^\star -\lambda_0))\right] - \epsilon_i^\star \ni 0 \quad \forall i \nonumber \;.
\end{align}
From line~\eqref{eq:alt_label_foc2} we can see that $UAU^T$ is diagonal, and hence $U$ consists of a set of eigenvectors of $A$. Defining $0 \leq a_1\leq a_2\leq \cdots \leq a_n$ to be the eigenvalues of $A$, we can then rewrite the problem as 
\begin{align}
\min_{\lambda_1,\dots, \lambda_n\in\mbr} \quad  &
\sum_{i=1}^n \lambda_ia_i \\
\mbox{subject to} \quad & \sum_{i=1}^n \lambda_i = n, \ \sum_{i=1}^n \min\left\{\frac{\lambda_i}{\lambda_0}, 1\right\} \geq k \label{eq:lam_constraint}, \ \lambda_i \geq 0 \quad \forall i \;. 
\end{align}

To solve this, consider a possible solution $\tilde\lambda_1,\dots, \tilde\lambda_n$. We will consider several cases. First, suppose there exists $i < j$ such that $\tilde\lambda_i, \tilde\lambda_j > \lambda_0$. Then define $\tilde \lambda_i'= \tilde \lambda_i + \tilde\lambda_j-\lambda_0$, $\tilde\lambda_j'=\lambda_0$, and $\tilde \lambda_m'=\tilde\lambda_m$ for $m\notin \{i,j\}$. Since $\tilde\lambda_i', \tilde\lambda_j' \geq \lambda_0$ and $\tilde\lambda_i'+\tilde\lambda_j'=\tilde\lambda_i+\tilde\lambda_j$ the constraints are still satisfied. Therefore, since $a_i\leq a_j$, $\sum_{i=1}^n \tilde\lambda_i'a_i \leq \sum_{i=1}^n \tilde\lambda_ia_i$, and so we know that there always exists an optimum with at most one $i$ such that $\lambda_i>\lambda_0$. Moreover, suppose that this index $i$ is larger than 1. Then we could set $\tilde \lambda_i'= \tilde \lambda_1$, $\tilde\lambda_1'=\tilde\lambda_i$, and $\tilde \lambda_m'=\tilde\lambda_m$ for $m\notin \{1,i\}$, thereby obtaining $\sum_{i=1}^n \tilde\lambda_i'a_i \leq \sum_{i=1}^n \tilde\lambda_ia_i$. Thus, there always exists an optimum $\lambda_1^\star,\dots, \lambda_n^\star$ with $\lambda_2^\star,\dots, \lambda_n^\star\leq \lambda_0.$

Next, suppose there exists $i<j$ such that $0<\tilde\lambda_i, \tilde\lambda_j < \lambda_0$. Then define $\tilde \lambda_i'=\tilde\lambda_i + \min\{\lambda_0-\tilde\lambda_i, \tilde\lambda_j\}$, $\tilde\lambda_j'=\tilde\lambda_j-\min\{\lambda_0-\tilde\lambda_i, \tilde\lambda_j\}$, and $\tilde \lambda_m'=\tilde\lambda_m$ for $m\notin \{i,j\}$. Since $\tilde\lambda_i', \tilde\lambda_j' \leq \lambda_0$ and $\tilde\lambda_i'+\tilde\lambda_j'=\tilde\lambda_i+\tilde\lambda_j$ the constraints are still satisfied. Therefore, since $a_i\leq a_j$, $\sum_{i=1}^n \tilde\lambda_i'a_i \leq \sum_{i=1}^n \tilde\lambda_ia_i$, we know that there always exists an optimum with at most one $i$ such that $0<\lambda_i<\lambda_0$. Now suppose that this $i$ is not the largest index such that $\lambda_i>0$. Then there exists an optimum with a $j>i$ such that $\lambda_j=\lambda_0$. Then we could set $\tilde \lambda_i'= \tilde \lambda_j$, $\tilde\lambda_j'=\tilde\lambda_i$, and $\tilde \lambda_m'=\tilde\lambda_m$ for $m\notin \{i,j\}$, thereby obtaining $\sum_{i=1}^n \tilde\lambda_i'a_i \leq \sum_{i=1}^n \tilde\lambda_ia_i$. Thus, there always exists an optimum $\lambda_1^\star,\dots, \lambda_n^\star$ with $\lambda_1^\star\geq \lambda_0$, $\lambda_2^\star,\dots, \lambda_{i-1}^\star=\lambda_0$, $0\leq\lambda_i^\star\leq \lambda_0$ for some $i$, and, if $i\neq n$, $\lambda_{i+1}^\star,\dots, \lambda_n^\star=0$.

Now from the constraint $\sum_{i=1}^n \min\left\{\frac{\lambda_i}{\lambda_0}, 1\right\} \geq k$, we can see that there must exist at least $k$ non-zero $\lambda_i$'s in the solution. 
If $n=k$, then we must have $\lambda_0=1$ and hence the optimum is given by $\lambda_1^\star,\dots, \lambda_k^\star=1$. Now consider the case where $n>k$.
Suppose there exists a solution $\tilde\lambda_1,\dots, \tilde\lambda_n$ such that $\tilde\lambda_{k+1} \neq  0$. Then since $\tilde\lambda_1,\dots, \tilde\lambda_k \geq \lambda_0$ we can set $\tilde\lambda_1'=\tilde\lambda_1+\tilde\lambda_{k+1}$, $\tilde\lambda_{k+1} = 0$, and $\tilde\lambda_j'=\tilde\lambda_j$ for $j\notin \{1,k+1\}$. This once again satisfies the constraints and $\sum_{i=1}^n \tilde\lambda_i'a_i \leq \sum_{i=1}^n \tilde\lambda_ia_i$. Therefore, there exists a solution such that $\lambda_1\geq \lambda_0$ and $\lambda_2,\dots, \lambda_k=\lambda_0$. In particular, a solution is $\lambda_1^\star=n-(k-1)\lambda_0$, $\lambda_2^\star,\dots, \lambda_k^\star=\lambda_0$.

In summary, the optima of this problem depend on the values of $k$ and $n$. In particular, we have:
\begin{itemize}
\item If $n>k$, an optimum  is given by $\lambda_1^\star=n-(k-1)\lambda_0$, $\lambda_2^\star,\dots, \lambda_k^\star=\lambda_0$, $\lambda_{k+1}^\star,\dots, \lambda_n^\star=0$.
\item If $n=k$, the optimum is given by $\lambda_1^\star,\dots, \lambda_k^\star=1$.
\end{itemize}
Returning to the original problem \eqref{eq:bh_relaxation}, we therefore have that an optimal $M$ is 
\begin{align*}
M^\star=\sum_{i=1}^n \lambda_i^\star u_iu_i^T\;,
\end{align*}
where $u_1,\dots, u_n$ are eigenvectors corresponding to the eigenvalues $a_1\leq a_2,\dots\leq a_n$ of $A$ and where $\lambda_1^\star,\dots, \lambda_n^\star$ are as defined above. 
\end{proof}

\subsection{Comparison to the XSDC relaxation}
We now compare the convex relaxation of the labeling subproblem presented in Section~\ref{sec:opt} to the relaxation proposed by \citet{bach2007}. As accommodating constraints on cluster labels is less natural in the latter relaxation, we compare the relaxations when training a LeNet-5 CKN on MNIST with no labeled data. Figure~\ref{fig:relax_comparison_acc} compares our matrix balancing method, the eigendecomposition method from the previous subsection, and the eigendecomposition method followed by $k$-means clustering. Prior to clustering, the rows of the eigenvector matrix were normalized to have unit $\ell_2$ norm. The value of $\lambda_0$ was chosen from the set $\{0.01n_b, 0.02n_b,\dots, 0.1n_b\}$, where $n_b$ is the size of a mini-batch, based on the performance on the validation set.

\begin{figure}[t]
	\begin{minipage}{0.48\linewidth}
		\centering
	\includegraphics[width=1.0\linewidth]{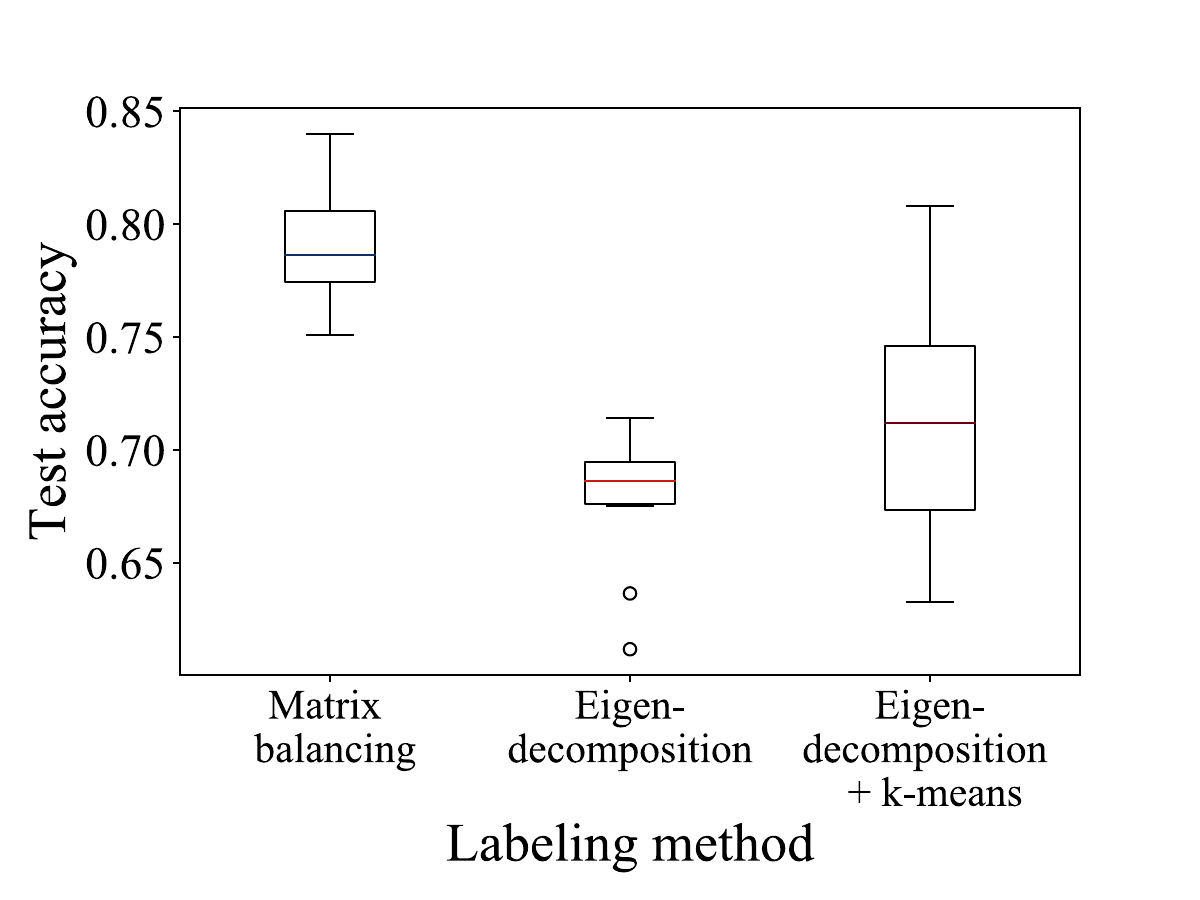}
	\caption{\label{fig:relax_comparison_acc} Performance of matrix balancing in comparison to eigendecomposition-based methods across 10 trials of training a LeNet-5 CKN on MNIST with no labeled data.}
	\end{minipage}\hspace{5pt}
	\begin{minipage}{0.48\linewidth}
	\centering
	\vspace{10pt}
	\includegraphics[width=1.0\linewidth]{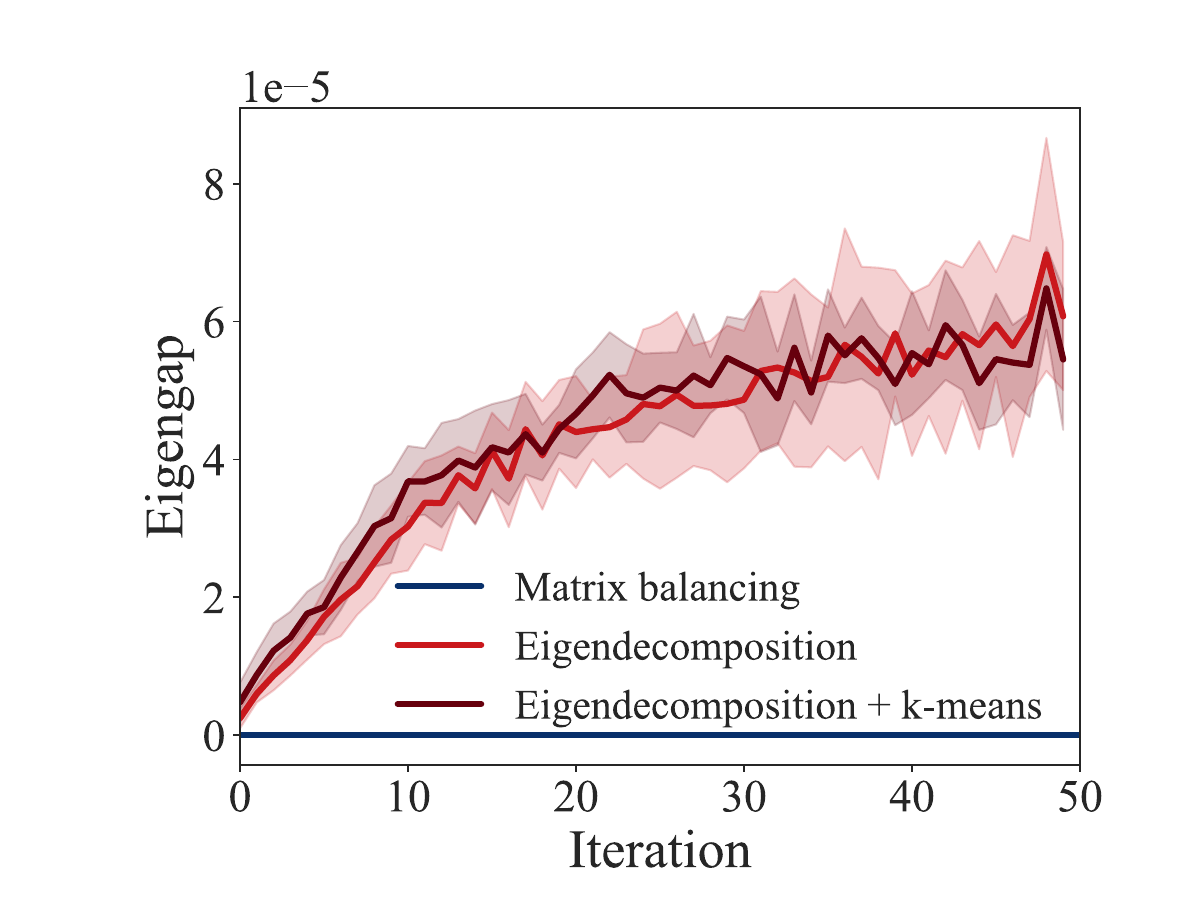}
	\caption{\label{fig:relax_comparison_gap} Evolution of the eigengap from matrix balancing in comparison to eigendecomposition-based methods across 10 trials of training a LeNet-5 CKN on MNIST with no labeled data. The error bands show one standard deviation from the mean.}
	\end{minipage}
\end{figure}

From Figure~\ref{fig:relax_comparison_acc} we can see that the convex relaxation used to derive the matrix balancing method is superior to the relaxations leading to the eigendecomposition-based methods. On average, matrix balancing performs 17\% better than the eigendecomposition method and 12\% better than the eigendecomposition method followed by $k$-means. This suggests that the constraint from the convex relaxation requiring the diagonal of $M$ to consist of all $1$'s and/or the constraint requiring all entries of $M$ to be positive are important for the performance of the labeling method.

In Figure~\ref{fig:relax_comparison_gap} we examine the eigengap of $A$ across iterations. The eigengap is defined as $\lambda_{k+1}-\lambda_{k}$, where $k$ is the number of classes and $\lambda_1\leq \dots\leq \lambda_n$ are the eigenvalues of $A$. As noted by \citet{meila2005}, having a larger eigengap makes the subspace spanned by the first $k$ eigenvectors of $A$ more stable to perturbations. From the figure we can see that the eigendecomposition-based methods tend to increase the eigengap as the learning proceeds. For the eigendecomposition method, the eigengap increased from $2\times10^{-6}$ to  $6\times10^{-5}$ on average after 50 iterations. Similarly, for the eigendecomposition method followed by $k$-means, the eigengap increased from  $5\times10^{-6}$ to  $5\times10^{-5}$ on average after 50 iterations. It is interesting to note that matrix balancing, which does not yield low-rank solutions $M^\star$, leads to eigengaps that are extremely small (on the order of $10^{-15}$) across the iterations. Nevertheless, it outperforms the eigendecomposition-based methods.

\section{Additional Experimental Details}
\label{app:training_details}
Here we provide additional details related to the training and the additional constraints we consider.

\subsection{Parameter tuning}
\label{sec:parameter_tuning}
The algorithm proposed in this paper and the models used require a large number of parameters to be set. Next, we discuss the choices for these parameters.

\paragraph{Fixed parameters.}
The parameters that are fixed throughout the experiments and not validated are as follows. The number of filters in the networks is set to 32 and the network's parameters $V$ are initialized layer-wise with 32 feature maps drawn uniformly at random from the output of the previous layer. The networks use the Nystr\"{o}m method to approximate the kernel at each layer. 
The regularization in the Nystr\"{o}m approximation is set to 0.001, and 20 Newton iterations are used to compute the inverse square root of the Gram matrix on the parameters $V_\ell$ at each layer $\ell$~\citep[Chapter 2]{jones2020}. 
The bandwidth is set to the median pairwise distance between the first 1000 observations for the single-layer networks. It is set to 0.6 for the convolutional networks. The batch size for both the labeled and unlabeled data is set to 4096 for Gisette and MAGIC and 1024 for MNIST and CIFAR-10 (due to GPU memory constraints).  The features output by the network $\phi$ are centered and normalized so that on average they have unit $\ell_2$ norm, as in \citet{mairal2014}. The initial training phase on just the labeled data is performed for 100 iterations, as the validation loss has typically started leveling off by 100 iterations.  The entropic regularization parameter $\entreg$ in the matrix balancing is set to the median absolute value of the entries in $A$. If this value results in divergence of the algorithm, it is multiplied by a factor of two until the algorithm converges. The value $n_\Delta$ is set to zero unless otherwise specified. The number of iterations of alternating minimization in the matrix balancing algorithm is set to 10.  The number of nearest neighbors used for estimating the labels on the unlabeled data is set to 1.

\paragraph{Hold-out validation.}
Due to the large number of hyperparameters, we tune them sequentially as follows when labeled data, and hence a labeled validation set, exists. 
First, we tune the penalty $\lambda$ on the classifier weights over the values $2^i$ for $i=-40,-39,\dots,0$. To do so, we train the classifier on only the labeled data using the initial random network parameters. We then re-validate this value every 100 iterations. Next, we tune the learning rate for the labeled data. 
For a small value $\reg=2^{-4}$ of the regularization parameter for the weights of the network, 
 we validate the fixed learning rate for the labeled data over the values $2^i$ for $i=-10,-9,\dots, 5$. 
The labeled and unlabeled data are then used to train the classifier used to compute the performance. 
For the unbalanced experiments on MNIST only we then tune the minimum and maximum size of the classes over the values $0.01b, 0.02b,\dots, 0.2b$, where $b$ is the batch size (fixing the semi-supervised learning rate to $2^{-5}$). For all other experiments we fix these values to $b/k$, where $k$ is the number of classes in the dataset.
We then tune the semi-supervised learning rate, again over the values $2^i$ for $i=-10,-9,\dots, 5$. 
For the single-layer networks we then tune $\reg$ over the values $2^i$ for $i=-10,-9,\dots, 10$.  For the convolutional networks we do not penalize the filters since they are constrained to lie on the sphere. 

\begin{figure*}[t!]
\centering
\includegraphics[trim={0.85cm 1.5cm 0cm 0cm},clip,width=0.8\linewidth]{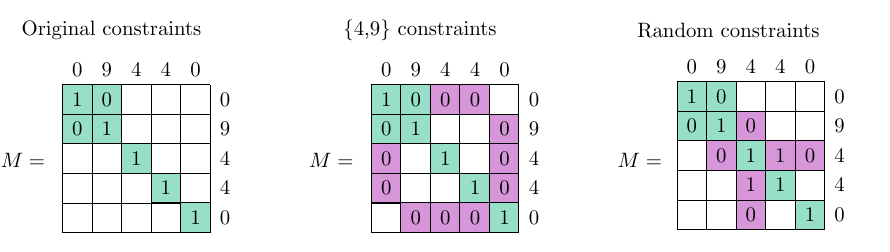}
\caption{\label{fig:constraints} Illustration of the kinds of additional constraints that were added. Green denotes the original constraints while purple denotes the constraints that were added. The numbers outside of the grids denote the true labels.}
\end{figure*}

When no labeled data exists we consider the hyperparameters in the same manner as during the hold-out validation. First we consider the values $2^i$ for $i=-10,-9,\dots, 5$ for the semi-supervised learning rate. Next we consider the values $2^i$ for $i=-40,-39,\dots,0$ for $\lambda$. Finally, if applicable, we consider the values $2^i$ for $i=-10,-9,\dots, 10$ for $\reg$.
We report the best performance observed on the test set. Developing a method for tuning the hyperparameters on an unlabeled validation set is left for future work.

\paragraph{Comparison details.}
\label{app:training_competitors}
In the comparisons we substitute our matrix balancing method with alternative labeling methods and retain the remainder of the XSDC algorithm. The pseudo-labeling code is our own, but we used code from \citet{caron2018} to implement the $k$-means version of deep clustering.\footnote{Their code may be found here: \url{https://github.com/facebookresearch/deepcluster}.} Two important details regarding the implementations are as follows. First, for pseudo-labeling when some of the data is labeled we estimate $W$ and $b$ based on the labeled data in the current mini-batch, as that is what is done in XSDC.  When labeled data is not present we estimate $W$ and $b$ based on the cluster assignments for the entire dataset. Second, for deep clustering we modify the dimension of the dimensionality reduction. In the original implementation the authors performed PCA, reducing the dimensionality of the features output by the network to 256. As the features output by the networks we consider have dimension less than 256, we instead keep the fewest number of components that account for 95\% of the variance. 

We perform the parameter tuning as follows. First, we follow the tuning procedure as detailed in Section~\ref{sec:parameter_tuning}. For pseudo-labeling there are no additional parameters to tune. However, for deep clustering there are two additional parameters to tune: the number of clusters in $k$-means and the number of iterations between cluster updates. During the initial parameter tuning stage these parameters are set to the true number of clusters $k$, and 50 iterations, respectively. Afterward we tune these two remaining parameters sequentially. We first tune the number of clusters over the values $k, 2k, 4k, 8k, 16k, 32k$ where $k$ is the true number of clusters. We then tune the number of iterations between cluster updates over the values $10, 25, 50, 100.$

\subsection{Additional constraints}
\label{sec:addl_constraints}
In one set of experiments we examine the effect of adding additional constraints. We consider two types of constraints: (1) constraints based on knowledge of whether the label was in the set $\{4,9\}$ or not; and (2) random correct must-link and must-not-link constraints among pairs of unlabeled observations and random correct must-not-link constraints between pairs of unlabeled and labeled observations. 

The two types of constraints are illustrated in Figure~\ref{fig:constraints}. Each grid point $(i,j)$, if filled, denotes whether observations $i$ and $j$ have the same label (1) or not (0). The true labels are the values outside of the grids. Green backgrounds correspond to knowing the labels corresponding to $(i,j)$. Purple backgrounds denote the additional known constraints. The left-most panel gives an example of an initial matrix $M$ in which the labels corresponding to the first two observations are known ($0$ and $9$). The second panel shows the entries we can fill in once we know whether each observation belongs to the set $\{4,9\}$. Finally, the third panel shows random correct constraints. The constraint at entry $(2,3)$ is a must-not-link constraint, whereas the constraint at entry $(3,4)$ is a must-link constraint.

\end{document}